%% file: neurips_2024.tex
\documentclass{article}

     \usepackage[preprint]{neurips_2024}

\usepackage[utf8]{inputenc} 
\usepackage[T1]{fontenc}    
\usepackage{hyperref}       
\usepackage{url}            
\usepackage{booktabs}       
\usepackage{amsfonts}       
\usepackage{nicefrac}       
\usepackage{microtype}      
\usepackage{xcolor}         
\usepackage[normalem]{ulem}
\useunder{\uline}{\ul}{}
\usepackage{wrapfig}
\usepackage{float} 

\usepackage{color-edits}
\addauthor{sw}{blue}
\addauthor{js}{orange}
\addauthor{jw}{brown}

\usepackage{graphicx}
\usepackage{subfigure}

\usepackage{amsmath}
\usepackage{amssymb}
\usepackage{mathtools}
\usepackage{algorithm}
\usepackage{algorithmic}
\usepackage{mathrsfs}
\usepackage{amsthm}

\usepackage[numbers]{natbib}

\newtheorem{theorem}{Theorem}[section]
\newtheorem{proposition}[theorem]{Proposition}
\newtheorem{lemma}[theorem]{Lemma}
\newtheorem{corollary}[theorem]{Corollary}
\newtheorem{definition}[theorem]{Definition}
\newtheorem{assumption}[theorem]{Assumption}
\newtheorem{remark}[theorem]{Remark}

\title{Bridging Multicalibration and Out-of-distribution Generalization Beyond Covariate Shift}

\author{%
 Jiayun Wu\thanks{Work completed as a research intern at Carnegie Mellon University.} \\
 Depart. of Computer Science \& Tech.\\
 Tsinghua University \\
 \texttt{wujy22@mails.tsinghua.edu.cn}
  \And
  Jiashuo Liu \\
  Depart. of Computer Science \& Tech. \\
  Tsinghua University \\
  \texttt{liujiashuo77@gmail.com} \\
  \AND
  \hspace{-0.8cm} Peng Cui \\
   \hspace{-0.8cm} Depart. of Computer Science \& Tech. \\
   \hspace{-0.8cm} Tsinghua University \\
  \hspace{-0.8cm} \texttt{cuip@tsinghua.edu.cn} \\
  \And
  Zhiwei Steven Wu \\
  School of Computer Science \\
  Carnegie Mellon University \\
  \texttt{zhiweiw@cs.cmu.edu} \\
}

\begin{document}

\maketitle

\begin{abstract}
\input{abstract.tex}
\end{abstract}

\input{introduction.tex}

\input{preliminary.tex}
\input{invariance.tex}
\input{structure.tex}
\input{algorithm.tex}
\input{experiment.tex}

\input{conclusion.tex}

\clearpage
\bibliography{example_paper}
\bibliographystyle{plainnat}

\clearpage


\appendix
\input{appendix.tex}

\end{document}

%% file: abstract.tex

We establish a new model-agnostic optimization framework for out-of-distribution generalization via multicalibration, a criterion that ensures a predictor is calibrated across a family of overlapping groups. Multicalibration is shown to be associated with robustness of statistical inference under covariate shift. We further establish a link between multicalibration and robustness for prediction tasks both under and beyond covariate shift. We accomplish this by extending multicalibration to incorporate grouping functions that consider covariates and labels jointly. This leads to an equivalence of the extended multicalibration and invariance, an objective for robust learning in existence of concept shift. We show a linear structure of the grouping function class spanned by density ratios, resulting in a unifying framework for robust learning by designing specific grouping functions. We propose MC-Pseudolabel, a post-processing algorithm to achieve both extended multicalibration and out-of-distribution generalization. The algorithm, with lightweight hyperparameters and optimization through a series of supervised regression steps, achieves superior performance on real-world datasets with distribution shift.

%% file: introduction.tex
\section{Introduction}
We revisit the problem of out-of-distribution generalization and establish new connections with multicalibration \cite{multicalibration}, a criterion originating from algorithmic fairness. Multicalibration is a strengthening of calibration, which requires a predictor $f$ to be correct \emph{on average} within each level set:
\[
\mathbb E[Y - f(X) \mid f(X)] = 0
\]
Calibration is a relatively weak property, as it can be satisfied even by the uninformative constant predictor $f(X) = \mathbb E[Y]$ that predicts the average outcome. More broadly, calibration provides only a marginal guarantee that does not extend to sub-populations. Multicalibration \cite{multicalibration} mitigates this issue by requiring the calibration to hold over a family of (overlapping) subgroups $\mathcal{H}$: for all $h\in \mathcal{H}$,
$$\mathbb E[(Y - f(X)) \, h(X) \mid f(X)] = 0$$
Multicalibration is initially studied as measure of subgroup fairness for boolean grouping functions~$h$, with $h(X)=1$ indicating $X$ is a member of group $h$~\citep{multicalibration}. Subsequently, \citet{omipredictor} and \citet{multiaccuracy} adopt a broader class of real-valued grouping functions that can identify sub-populations through reweighting.
The formulation of real-valued grouping function has enabled surprising connections between multicalibration and distribution shifts.
Prior work~\citep{UniversalAdaptability,uncertain} studied how distribution shift \emph{affects} the measure of multicalibration, with a focus on covariate shift where the relationship between $X$ and $Y$ remains fixed. \citet{UniversalAdaptability} show that whenever the set of real-valued grouping functions $\mathcal H$ includes the density ratio between the source and target distributions, a multicalibrated predictor with respect to the source remains calibrated in the shifted target distribution.


Our work substantially expands the connections between multicalibration and distribution shifts. At a high level, our results show that robust prediction under distribution shift can actually be \emph{facilitated} by multicalibration. We extend the notion of multicalibration by incorporating grouping functions that simultaneously consider both covariates $X$ and outcomes $Y$. This extension enables us to go beyond covariate shift and account for concept shift, which is prevalent in practice due to spurious correlation, missing variables, or confounding~\citep{whyshift}. 



\textbf{Our contributions.} Based on the introduction of joint grouping functions, we establish new connections between our extended multicalibration notion and algorithmic robustness in the general setting of out-of-distribution generalization, where the target distribution to assess the model is different from the source distribution to learn the model. 

 1. We first revisit the setting of covariate shift and show multicalibration implies Bayes optimality under covariate shift, provided a sufficiently rich class of grouping functions. Then, in the setting of concept shifts, we show the equivalence of multicalibration and invariance~\citep{IRM}, a learning objective to search for a Bayes optimal predictor $\mathbb E[Y|\Phi(X)]$ under a representation over features $\Phi(X)$, even though $\mathbb E[Y|X]$ is different across target distributions.
 We show correspondence between an invariant representation $\Phi(X)$ and a multicalibrated predictor $\mathbb E[Y|\Phi(X)]$, with a grouping function class containing all density ratios of target distributions and the source distribution. 

 2. As part of our structural analysis of the new multicalibration concept, we investigate the maximal grouping function class that allows for a nontrivial multicalibrated predictor. For traditional covariate-based grouping functions, the Bayes optimal predictor $f(X)=\mathbb E[Y|X]$ is always multicalibrated, which is no longer the case for joint grouping functions. We show the maximal grouping function class is a linear space spanned by the density ratio of the target distributions where the predictor is invariant. As a structural characterization of distribution shift, this leads to an efficient parameterization of the grouping functions by linear combination of a spanning set of density ratios. The spanning set can be flexibly designed to incorporates implicit assumptions of various methodologies for robust learning, including multi-environment learning~\citep{IDGM} and hard sample learning~\citep{JTT}.

3. We devise a post-processing algorithm to multicalibrate predictors and simultaneously producing invariant predictors. As a multicalibration algorithm, we prove its convergence under Gaussian distributions of data and certify multicalibration upon convergence. As a robust learning algorithm, the procedure is plainly supervised regression with respect to models' hypothesis class and grouping function class, introducing an overhead of linear regression. This stands out from heavy optimization techniques for out-of-distribution generalization, such as bi-level optimization~\citep{MAML, gdro} and multi-objective learning~\citep{IB-IRM,IRM,MIP}, which typically involves high-order gradients~\citep{fishr}. The algorithm introduces no extra hyperparameters. This simplifies model selection, which is a significant challenge for out-of-distribution generalization since validation is unavailable where the model is deployed~\citep{domainBed}. Under the standard model selection protocol of DomainBed~\citep{domainBed}, the algorithm achieves superior performance to existing methods in real-world datasets with concept shift, including porverty estimation~\citep{poverty}, personal income prediction~\citep{income} and power consumption~\citep{shift,shift2} prediction.

%% file: preliminary.tex
\section{Multicalibration and Bayes Optimality under Covariate Shift}
\subsection{Multicalibration with Joint Grouping Functions}
\paragraph{Settings of Out-of-distribution Generalization} We are concerned with prediction tasks under distribution shift, where covariates are denoted by a random vector $X\in \mathcal X$ and the target by $Y \in \mathcal Y$. 
The values taken by random variables are written by $x,y$ in lowercase. 
We have an \emph{uncertainty set} of absolutely continuous probability measures, denoted by $\mathcal P(X,Y)$, where there is an accessible source measure $P_S\in\mathcal P$ and unknown target measure $P_T\in\mathcal P$. We use \emph{capital} letters such as $P$ to denote a single probability measure and \emph{lowercase} letters such as $p$ to denote its probability density function. 
We define predictors as real-valued functions $f: \mathcal X \rightarrow \mathcal Y$, which is learned in the source distribution $P_S$ and assessed in the target distribution $P_T$.
 Given a loss function $\ell:\mathcal Y \times \mathcal Y \rightarrow \mathbb R$, we evaluate the average risk of a predictor $f$ w.r.t. a probability measure $P$, defined by $R_P(f) := \mathbb E_P[\ell(f(X),Y)]$. 
We focus on the setting with $\mathcal Y=[0,1]$ and $\ell(\hat y,y) = (\hat y-y)^2$ in our theoretical analyses. 
 
We propose a new definition of $\ell_2$ approximate multicalibration with joint grouping functions.

\begin{definition}[Multicalibration with Joint Grouping Functions] 
\label{def:mc}
For a probability measure $P(X,Y)$ and a predictor $f$, let $\mathcal H \subset \mathbb R^{\mathcal X \times\mathcal Y}$ be a real-valued \emph{grouping function class}. We say that $f$ is $\alpha$-approximately $\ell_2$ multicalibrated w.r.t. $\mathcal H$ and $P$ if for all $h \in \mathcal H$:
\begin{align}
	\label{eq:mc_err}
	K_2(f, h, P)  
	= \int \Big(\mathbb E_P\left[h(X,Y)(Y-v)\big|f(X)=v\right]\Big)^2 dP_{f(X)}(v) 
	\leq \alpha.
\end{align}
$P_{f(X)}(v)=P(f^{-1}(v))$ is the pushforward measure.   We say $f$ is $\alpha$-approximately calibrated for $h \equiv 1$. We say $f$ is multicalibrated (calibrated) for $\alpha=0$. If the grouping function is defined on $X$, which implies $h(x,y_1)=h(x,y_2)$ for any $x\in\mathcal X$ and $y_1,y_2\in\mathcal Y$, we abbreviate $h(x,\cdot)$ by $h(x)$.
\end{definition} 
In the special case of a constant grouping function $h\equiv 1$, $K_2(f,1,P)$ recovers the overall calibration error. 
For boolean grouping functions defined on $X$~\citep{multicalibration}, $K_2(f,h,P)$ computes the calibration error of the subgroup with $h(x)=1$. For real-valued grouping functions defined on $X$~\citep{omipredictor, multiaccuracy}, $K_2(f,h,P)$ evaluates a reweighted calibration error, whose weights $h(x)$ are proportional to the likelihood of a sample belonging to the subgroup. Most importantly, we propose an extended domain of grouping functions defined on covariates and outcomes jointly, which is useful for capturing more general distribution shifts. Multicalibration error quantifies the maximal calibration error for all subgroups associated with grouping functions in $\mathcal H$. We will discuss multicalibration with covariate-based grouping functions in this section and joint grouping functions in next section.

\subsection{Multicalibration Implies Bayes Optimality under Covariate Shift}

In this subsection we focus on grouping functions $h(x)$ defined on covariates. We will prove approximately multicalibrated predictors simultaneous approaches Bayes optimality in each target distribution with covariate shift, 
bridging the results of \citet{UniversalAdaptability} and \citet{MCBoost}.
To recap, \citet{UniversalAdaptability} studies multicalibration under covariate shift and shows that a multicalibrated predictor remains calibrated in target distribution for a sufficiently large grouping function class. Further, it is shown that multicalibration predictors remain multicalibrated under covariate shift~\citep{UniversalAdaptability, uncertain}, assuming the grouping function class $\mathcal H$ is closed under some transformation by density ratios (Assumption~\ref{assum:sufficient_grouping}.1).
Second, \citet{MCBoost} shows multicalibration implies Bayes optimal accuracy~\citep{MCBoost}, assuming $\mathcal H$ satisfies a weak learning condition (Assumption~\ref{assum:sufficient_grouping}.2). 
Detailed discussion on other related works is deferred to section \ref{sec:app_related} in the appendix.

\begin{assumption}[Sufficiency of Grouping Function Class (informal, see Assumption~\ref{app_assum:app_sufficient_grouping})]
\label{assum:sufficient_grouping}
\;
\vskip 0.01in
1. (Closure under Covariate Shift)
For a set of probability measures $\mathcal P(X)$ containing the source measure $P_S(X)$, $h\in\mathcal H$ implies $p/p_S\cdot h\in\mathcal H$ for any density function $p$ of distributions in $\mathcal P$.

2. (($\gamma, \rho$)-Weak Learning Condition)
For any $P \in \mathcal P(X)P_S(Y|X) \equiv \{P'(X) P_S(Y\mid X) : P' \in \mathcal{P}\}$ with the source measure $P_S(Y|X)$, and every subset $G \subset \mathcal X$ with $P(X\in G)>\rho$, if the Bayes optimal predictor $\mathbb E_P[Y|X]$ has lower risk than the constant predictor $\mathbb E_P[Y|X\in G]$ by a margin $\gamma$, there exists a predictor $h\in \mathcal H$ that is also better than the constant predictor with the margin $\gamma$.
\end{assumption}

\begin{theorem}[Risk Bound under Covariate Shift]
\label{theo:regret_covariate}
	For a source measure $P_S(X,Y)$ and a set of probability measures $\mathcal P(X)$ containing $P_S(X)$, given a predictor $f:\mathcal X \rightarrow [0,1]$ with finite range $m:=|\text{Range}(f)|$, consider a grouping function class $\mathcal H$ closed under affine transformation and satisfying Assumption \ref{assum:sufficient_grouping} with $\rho = \gamma/m$. If $f$ is $\frac{\gamma^6}{256m^2}$-approximately $\ell_2$ multicalibrated w.r.t $P_S$ and $\mathcal H_1 := \left\{ h \in \mathcal H:\max_{x \in \mathcal X}h(x)^2 \leq 1 \right\}$, then for any target measure $P_T \in \mathcal P(X)P_S(Y|X)$,
\begin{align}
	R_{P_T}(f) \leq \inf_{f^*:\mathcal X\rightarrow [0,1]} R_{P_T}(f^*) + 3\gamma.
\end{align}
\end{theorem}
\begin{remark}	Following prior work in multicalibration~\citep{MCBoost,uncertain}, we study functions $f$ with finite cardinality, which can be obtained by discretization.
\end{remark}


%% file: invariance.tex
\section{Multicalibration and Invariance under Concept Shift}
\label{sec:mc_inv}
Theorem \ref{theo:regret_covariate} shows multicalibration implies Bayes optimal accuracy for target distributions under covariate shift. However, in practical scenarios, there are both marginal distribution shifts of covariates ($X$) and \emph{concept shift} of the conditional distributions ($Y|X$). Concept shift is especially prevalent in tabular data due to missing variables and confounding~\citep{whyshift}.
In order to go beyond covariate shift, we will focus on grouping functions defined on covariates and outcomes jointly. We show that multicalibration notion w.r.t. joint grouping functions is equivalent to invariance, a criterion for robust prediction under concept shift.
Extending the robustness of multicalibration to general shift is non-trivial. The fundamental challenge is that there is no shared predictor that is generally optimal in each target distribution because the Bayes optimal predictor varies for different $Y|X$ distributions.  
As a first step, we show multicalibrated predictors w.r.t. joint grouping functions are robust as they are optimal over any post-processing functions in each target distribution.
\begin{theorem}[Risk Bound under Concept Shift]
\label{theo:mc_inv}
For a set of absolutely continuous probability measures $\mathcal P(X,Y)$ containing the source measure $P_S(X,Y)$, consider a predictor $f:\mathcal X\rightarrow [0,1]$. Assume the grouping function class $\mathcal H$ satisfies the following condition:
\begin{align}
	\label{eq:density_ratio}
	\mathcal H \supset \left\{ h(x,y)=\frac{p(x,y)}{p_S(x,y)} \Big| P \in \mathcal P(X,Y) \right\}.
\end{align}
If $f$ is $\alpha$-approximately $\ell_2$ multicalibrated w.r.t. $\mathcal H$ and $P_S$, then for any measure $P \in \mathcal P(X,Y)$,
\begin{align}
\label{eq:mc_inv_1}
	R_P(f) \leq \inf_{g:[0,1]\rightarrow [0,1]} R_P(g\circ f) + 2\sqrt \alpha.
\end{align}
\end{theorem}
The theorem shows an \emph{approximately multicalibrated} predictor on the source \emph{almost cannot be improved by post-processing} for each target distribution. To ensure such robustness, the grouping function class must include all density ratios between target and source measures, which are functions over $\mathcal X \times \mathcal Y$. This characterization of robustness in terms of post-processing echoes with Invariant Risk Minimization (IRM) \citep{IRM}, a paradigm for out-of-distribution generalization with $Y|X$ shift. However, their analysis focuses on representation learning.
\begin{definition}[Invariant Predictor] 
\label{def:invariant_pred}
Consider data selected from multiple environments in the set $\mathscr E$ where the probability measure in an environment $e \in \mathscr E$ is denoted by $P_e(X,Y)$. Denote the representation over covariates by a measurable function $\Phi(x)$. We say that $\Phi$ elicits an $\alpha$-approximately invariant predictor $g^*\circ \Phi$ across $\mathscr E$ if there exists a function $g^* \in \mathcal G := \left\{g:\text{supp}(\Phi) \rightarrow [0,1]\right\}$ such that for all $e \in \mathscr E$:
\begin{align}
	\label{eq:inv}
	R_{P_e}(g^*\circ \Phi) \leq \inf_{g \in \mathcal G} { R_{P_e}(g\circ \Phi)} + \alpha.
\end{align}
\end{definition}
\begin{remark}
	(1) Predictors in $\mathcal G$ take a representation $\Phi$ extracted from the covariates as input. For a general predictor $f(x)$, if we take $\Phi(x)=f(x)$ and $g^*$ as an identity function, Equation~\ref{eq:inv} reduces to the form of Equation~\ref{eq:mc_inv_1}.  Therefore, $f$ in Equation~\ref{eq:mc_inv_1} is a $2\sqrt \alpha$-approximately invariant predictor across environments collected from the uncertainty set $\mathcal P$. 
	(2) We give an approximate definition of invariant predictors, which recovers the original definition~\citep{IRM} when $\alpha=0$. In this case, there exists a shared Bayes optimal predictor $g^\star$ across environments, taking $\Phi$ as input. This implies $\mathbb E_{e_1}[Y|\Phi]=\mathbb E_{e_2}[Y|\Phi]$ almost surely for any $e_1,e_2$.
\end{remark}
IRM searches for a representation such that the optimal predictors upon the representation are \emph{invariant} across environments. Motivated from causality, the interaction between outcomes and their causes are also assumed invariant, so IRM learns a representation of causal variables for stable prediction. 
We extend Theorem~\ref{theo:mc_inv} to representation learning and prove equivalence between multicalibrated and invariant predictors.
\begin{theorem}[Equivalence of Multicalibration and Invariance]
\label{theo:equiv_mc_inv}
Assume samples are drawn from an environment $e\in\mathscr E$ with a prior $P_S(e)$ such that $\sum_{e\in\mathscr E}P_S(e)=1$ and $P_S(e)>0$.  The overall population satisfies $P_S(X,Y)=\sum_{e\in\mathscr E}P_e(X,Y)P_S(e)$ where $P_e(X,Y)$ is the environment-specific absolutely continuous measure. With a measurable function $\Phi(x)$, define a function class $\mathcal H$ as:
\begin{align}
	\label{eq:equiv_mc_inv_density1}
	\mathcal H:=\left\{ h(x,y)=\frac{p_e(x,y)}{p_S(x,y)} \Big| e \in \mathscr E \right\}.
\end{align}
1. If there is a bijection $g^\star:\text{supp}(\Phi) \rightarrow [0,1]$ such that $g^\star \circ \Phi$ is $\alpha$-approximately $\ell_2$ multicalibrated w.r.t. $\mathcal H$ and $P_S$, then $\Phi$ elicits an $2\sqrt \alpha$-approximately invariant predictor $g^\star\circ \Phi$ across $\mathscr E$.

2. If there is $g^\star:\text{supp}(\Phi) \rightarrow [0,1]$ such that $\Phi$ elicits an $\alpha$-approximately invariant predictor $g^\star\circ \Phi$ across $\mathscr E$, then $g^\star \circ \Phi$ is $\sqrt {\alpha/D}$-approximately $\ell_2$ multicalibrated w.r.t. $\mathcal H$ and $P_S$, where $D=\min_{e\in \mathscr E}P_S(e)$.
\end{theorem}
\begin{remark}
	(1) In the first statement, assuming $g^\star$ is a bijection avoids degenerate cases where $\Phi$ contains redundant information. For example, every predictor $g^\star \circ \Phi(X)$ upon representation $\Phi$ equals $g^\star(\Phi)\circ \mathbb I(X)$ upon representation $X$. Confining $g^\star$ to bijections ensures some unique decomposition into predictors and representations.  
	(2) \citet{clove} proves equivalence between exact invariance and simultaneous calibration in each environment. We strengthen their result to show multicalibration on a single source distribution suffices for invariance. Moreover,  our results can be directly extended beyond their multi-environment setting to a general uncertainty set of target distributions, by the mapping between grouping functions and density ratios.
 Further, our theorem is established for both exact and approximate invariance.
\end{remark}
The theorem bridges \emph{approximate multicalibration} with \emph{approximate invariance} for out-of-distribution generalization beyond covariate shift. The equivalence property indicates that the density ratios of target and source distributions constitute the \emph{minimal grouping function class} required for robust prediction in terms of invariance. 

%% file: structure.tex
\section{Structure of Grouping Function Classes}

Section~\ref{sec:mc_inv} inspires one to construct richer grouping function classes for stronger generalizability. However, fewer predictors are multicalibrated to a rich function class. For covariate-based grouping functions, the Bayes optimal $\mathbb E[Y|X]$ is always multicalibrated. But the existence of a multicalibrated predictor is nontrivial for joint grouping functions. For example, when we take the grouping function $h(x,y)=y$, a multicalibrated predictor $f$ implies $f(X)=\mathbb E[Y|f(X),h(X, Y)]=Y$, which is impossible for regression with label noise. In this section, we first study the \emph{maximal grouping function class} that is feasible for a multicalibrated predictor. Then, we will leverage our structural results to inform the design of grouping functions.

\subsection{Maximal Grouping Function Space}
 We focus on continuous grouping functions defined on a compact set $\mathcal X\times \mathcal Y \subset \mathbb R^{d+1}$, i.e., $h\in C(\mathcal X\times \mathcal Y)$, and consider absolutely continuous probability measures supported on $\mathcal X\times\mathcal Y$ with continuous density functions. Our first proposition shows that the maximal grouping function class for any predictor is a linear space.

\begin{proposition}[Maximal Grouping Function Class]
\label{prop:linear_space}
	Given an absolutely continuous probability measure $P_S(X,Y)$ and a predictor $f:\mathcal X\rightarrow[0,1]$, define the maximal grouping function class that $f$ is multicalibrated with respect to:
	\begin{align}
		\mathcal H_f := \{ h\in C(\mathcal X\times \mathcal Y) : K_2(f,h,P_S)=0 \}.
	\end{align}
	Then $\mathcal H_f$ is a linear space.
\end{proposition}

In the following, we further analyze the spanning set of maximal grouping function classes for nontrivial predictors which are at least calibrated.
\begin{theorem}[Spanning Set]
\label{theo:maximal_grouping_phi}
	Consider an absolutely continuous probability measure $P_S(X,Y)$ and a calibrated predictor $f:\mathcal X \rightarrow [0,1]$. Then its maximal grouping function class $\mathcal H_f$ is given by:
	\begin{align}
		\mathcal H_{f} 
		 = \mathrm{span}\left\{ \frac{p(x,y)} {p_S(x,y)}:\text{$p$ is continuous and } R_P(f)=\inf_{g:[0,1]\rightarrow [0,1]} R_P(g\circ f) \right\}.
	\end{align}
\end{theorem}
A predictor's maximal grouping function class is spanned by density ratios of target distributions where the predictor is invariant. Correspondingly, Theorem~\ref{theo:mc_inv} gives the minimal grouping function class, comprised of density ratios between target and source distributions, in order to ensure $f(x)$ is an invariant predictor. In contrast, Theorem~\ref{theo:maximal_grouping_phi} states the maximal grouping function class for $f(x)$ is exactly the linear space spanned by those density ratios. 
Next, we further investigate sub-structures of the maximal grouping function class. 
We focus on the representation learning setting of IRM.
\begin{theorem}[Decomposition of Grouping Function Space]
\label{theo:decomposeh1h2}
	Consider an absolutely continuous probability measure $P_S(X,Y)$ and a measurable function $\Phi:\mathbb R^{d} \rightarrow \mathbb R^{d_\Phi}$ with ${d_\Phi}\in Z^+$. We define the Bayes optimal predictor over $\Phi$ as $f_{\Phi}(x)=\mathbb E_{P_S}[Y|\Phi(x)]$. We abbreviate $\mathcal H_{f_\Phi}$ with $\mathcal H_{\Phi}$. 
Then $\mathcal H_\Phi$ can be decomposed as a Minkowski sum of $\mathcal H_{1,\Phi}+\mathcal H_{2,\Phi}$.
	\begin{align}
		\label{eq:theo_decompose_1}
		\mathcal H_{1,\Phi} 
		&= \mathrm{span}\left\{ \frac{p(\Phi,y)} {p_S(\Phi,y)}: \text{$p$ is continuous and } R_P(f_\Phi)=\inf_{g:[0,1]\rightarrow [0,1]} R_P(g\circ f_\Phi)\right\}. \\
		\mathcal H_{2,\Phi} 
		&= \mathrm{span}\left\{ \frac{p(x|\Phi,y)} {p_S(x|\Phi,y)}: \;\text{$p$ is continuous}\right\}.
	\end{align}
1. If a predictor $f$ is multicalibrated with $\mathcal H_{1,\Phi}$, then $R_{P_S}(f) \leq R_{P_S}(f_\Phi)$.

2. $f_\Phi$ is an invariant predictor elicited by $\Phi$ across a set of environments $\mathscr E$ where $P_e(\Phi,Y)=P_S(\Phi,Y)$ for any $e\in\mathscr E$.  If a predictor $f$ is multicalibrated with $\mathcal H_{2,\Phi}$,  then $f$ is also an invariant predictor across $\mathscr E$ elicited by some representation.
\end{theorem}
\begin{remark}
$\mathcal H_{1,\Phi}$ and $\mathcal H_{2,\Phi}$ contain functions defined on $x,\Phi(x),y$ which can both be rewritten as functions on $x,y$ by variable substitution. Thus, $\mathcal H_{1,\Phi},\mathcal H_{2,\Phi}$ are still subspaces of grouping functions. $\mathcal H_{1,\Phi}$ is spanned by the density ratio of $P(\Phi,Y)$ where the Bayes optimal predictor over $\Phi$ must be invariant on the distribution of $P$. 
$\mathcal H_{2,\Phi}$ is spanned by general density ratio of $P(X|\Phi,Y)$.
\end{remark}
Multicalibration w.r.t. $\mathcal H_{1,\Phi}$ ensures at least the \emph{accuracy} of the Bayes optimal predictor on $\Phi$, and multicalibration w.r.t. $\mathcal H_{2,\Phi}$ ensures at least the \emph{invariance} of this predictor. However, we show in Proposition~\ref{prop:app_monoticity} that sizes of two subspaces are negatively correlated. When $\Phi$ is a variable selector, $\mathcal H_{1,\Phi}$ expands with more selected covariates while $\mathcal H_{2,\Phi}$ shrinks. By choosing a combination of $\mathcal H_{1,\Phi}$ and $\mathcal H_{2,\Phi}$, we strike a balance between accuracy and invariance of the multicalibrated predictor.
\begin{proposition}[Monotonicity]
\label{prop:monoticity}
	Consider $X\in\mathbb R^d$ which could be sliced as $X=(\Phi,\Psi)^T$ and $\Phi=(\Lambda,\Omega)^T$. Define $\mathcal H_{1,\Phi}' := \{ h(\Phi(x)) \in C(\mathcal X\times \mathcal Y) \}$. $\mathcal H_{1,X}'$ and $\mathcal H_{1,\Lambda}'$ are similarly defined.
	We have:
	
1. $\mathcal H_{1,X}' \supset \mathcal H_{1,\Phi}' \supset \mathcal H_{1,\Lambda}' \supset \mathcal H_{1,\emptyset}' = \{ C\}.$

2. $\{C\}=\mathcal H_{2,X} \subset H_{2,\Phi} \subset \mathcal H_{2,\Lambda} \subset \mathcal H_{2,\emptyset}.$

$C$ is a constant value function.
\end{proposition}
\subsection{Design of Grouping Function Classes}
The objective of a robust learning method can be represented by a tuple consisting of an assumption about the boundary of distribution shift and a metric of robustness. Multicalibration is equivalent to invariance as a metric of robustness, while the grouping function class provides a unifying view for assumptions over potential distribution shift. Given any uncertainty set of target distributions $\mathcal P$, Theorem~\ref{theo:maximal_grouping_phi} implies an efficient and reasonable construction of grouping functions as linear combinations of density ratios from $\mathcal P$. We implement two designs of grouping functions for the learning setting with and without environment annotations respectively.

\textbf{From Environments}\;\;
If samples are drawn from multiple environments and the environment annotations are available, we assume the uncertainty set as the union of each environment's distribution $P_e$. This completely recovers IRM's objective, but we approach it with a different optimization technique in the next section.
Taking pooled data as the source $S$, density ratios spanning the grouping function class are $h_e(x,y)=p_e(x,y)/p_S(x,y)=p_S(e|x,y)/p_S(e)$, where $p_S(e|x,y)$ is estimated by an environment classifier.  Then a grouping function can be represented as a linear combination of $h_e$:
\begin{align}
\label{eq:h_irm}
	h(x,y) = \sum_{e\in \mathscr E}\lambda_e p_S(e|x,y) ,\;\;\lambda_e\in\mathbb R.
\end{align}
\textbf{From Hard Samples}\;\;
When data contains latent sub-populations without annotations, the uncertainty set can be constructed by identifying sub-populations. Hard sample learning~\citep{TERM,focalLoss,JTT} suggests the risk is an indicator for sub-population structures. Samples from the minority sub-population $M$ are more likely to have high risks. For example, JTT~\citep{JTT} identified the minority subgroup using a risk threshold of a trained predictor $f_{id}$. We adopt a continuous grouping by assuming $P_S(X,Y\in M)\propto (f_{id}(X)-Y)^2$. We construct the uncertainty set as the union of the source $S$ and minority sub-population $M$, resulting in a grouping function represented as:
\begin{align}
\label{eq:h_jtt}
	h(x,y) = \lambda_M (f_{id}(x)-y)^2 +\lambda_S,\;\;\lambda_M,\lambda_S\in\mathbb R.
\end{align}
Another design utilizing Distributionally Robust Optimization's assumption~\cite{alpha_dro} is in section~\ref{sec:app_grouping}.

%% file: algorithm.tex
\section{MC-PseudoLabel: An Algorithm for Extended Multicalibration}
In this section, we introduce an algorithm for multicalibration with respect to joint grouping functions. Simultaneously, the algorithm also provides a new optimization paradigm for invariant prediction under distribution shift.
The algorithm, called MC-PseudoLabel, post-processes a trained model by supervised learning with pseudolabels generated by grouping functions.  As shown in  Algorithm~\ref{alg}, given a predictor function class $\mathcal F$ and a dataset $D$ with an empirical distribution $\hat {P}_D(X,Y)$, a regression oracle $A$ solves the optimization: $A_{\mathcal F}(D) = \arg \min_{f\in\mathcal F} R_{\hat {P}_D}(f)$. We take as input a model $f_0$, possibly trained by Empirical Risk Minimization. $f_0$ has a finite range following conventions of prior work in multicalibration~\citep{MCBoost}. For continuous predictors, we discretize the model output and introduce a small rounding error (see section~\ref{app_discretize}). For each iteration, the algorithm performs regression with grouping functions on each level set of the model. The prediction of grouping functions rectify the uncalibrated model and serves as pseudolabels for model updates. 

\begin{small}
\begin{algorithm}[ht]
\caption{MC-PseudoLabel}
\label{alg}
\begin{algorithmic}[1]
\REQUIRE A dataset \(D=(D_x,D_y)\), a grouping function class \(\mathcal H\), a predictive function class \(\mathcal F\).
\STATE \(t \leftarrow 0;\)
\STATE \(f_0 \leftarrow \text{Initialization};\) \;\;\;\;\COMMENT{\textit{For example, models trained with ERM.}}
\STATE \(m \leftarrow |\text{Range}(\text{Discretize}(f_0))|;\) 
\REPEAT
    \STATE \(f_t' \leftarrow \text{Round}(f_t; m):=\arg\min_{v\in[1/m]} |f_t(x)-v| ;\)
    \STATE \( err_t = \mathbb{E}_{x,y\sim D}[(f_t'(x) -y)^2]; \)
    \FOR{each \(v \in [1/m]\)} 
        \STATE \(D_v^t \leftarrow D| f_t'(x) = v; \)
        \STATE \(h_v^t(x,y) \leftarrow A_{\mathcal H}(D_v^t);\) \;\;\;\;\COMMENT{\textit{Regression on level sets with grouping functions.}}
    \ENDFOR
    \STATE \(\tilde f_{t+1}(x,y) \leftarrow \sum_{v\in [\frac{1}{m}]} 1_{\{f_t'(x)=v\}} \cdot h_v^t(x,y);\) \;\;\;\;\COMMENT{\textit{Generate pseudolabels.}}
    \STATE \(\tilde{err}_{t+1} \leftarrow \mathbb{E}_{x,y\sim D}[(\tilde f_{t+1}(x,y) - y)^2];\)
    \STATE \(D_{t+1} \leftarrow (D_x, \tilde f_{t+1}(D));\) 
    \STATE \(f_{t+1}(x) \leftarrow A_{\mathcal F}(D_{t+1});\) \;\;\;\;\COMMENT{\textit{Update the model with pseudolabels.}}
    \STATE \(t \leftarrow t+1;\)
\UNTIL{$err_{t-1} - \tilde{err}_t$ stops decreasing.} 
\ENSURE \(f_{t-1}'(x).\)
\end{algorithmic}
\end{algorithm}
\end{small}
Since we regress $Y$ with grouping functions defined on $Y$, a poor design of groupings violating Theorem~\ref{theo:maximal_grouping_phi} can produce trivial outputs. For example, if grouping functions contain $h(x,y)=y$, then $err_{t-1}-\tilde{err}_t$ never decreases and the algorithm outputs $f_0$, because there does not exist a multicalibrated predictor. However, the algorithm certifies a multicalibrated output if it converges.
\begin{theorem}[Certified Multicalibration]
\label{theo:stopping_criteria}
	In Algorithm \ref{alg}, for $\alpha,B>0$, if $err_{t-1} - \tilde{err}_t \leq \frac{\alpha}{B}$,
 the output $f_{t-1}'(x)$ is $\alpha$-approximately $\ell_2$ multicalibrated w.r.t. $\mathcal H_B=\{h\in \mathcal H:\sup h(x,y)^2\leq B\}$.
\end{theorem}
MC-PseudoLabel reduces to LSBoost~\citet{MCBoost}, a boosting algorithm for multicalibration if $\mathcal H$ only contains covariate-based grouping functions. In this case, Line 14 of Algorithm~\ref{alg} reduces to $f_{t+1}(x)=\tilde{f}_{t+1}(x,\cdot)$ where $\tilde{f}_{t+1}$ does not depend on $y$. For joint grouping functions, since $\tilde{f}_{t+1} \in \mathbb R^{\mathcal X\times \mathcal Y}$, we project it to models' space of $\mathbb R^{\mathcal X}$ by learning the model with $\tilde{f}_{t+1}$ as pseudolabels. The projection substantially changes the optimization dynamics. LSBoost constantly decreases risks of models, due to $R_{\hat P_D}(f_{t+1}) = R_{\hat P_D}(\tilde f_{t+1})<R_{\hat P_D}(f_t)$. The projection step disrupts the monotonicity of risks, implying that
MC-Pseudolabel can output a predictor with a higher risk than input. This is because multicalibration with joint grouping functions implies balance between accuracy and invariance, as is discussed in Theorem~\ref{theo:decomposeh1h2}.
The convergence of LSBoost relies on the monotonicity of risks, which is not applicable to MC-Pseudolabel. We study the algorithm's convergence in the context of representation learning. Assume we are given a grouping function class $\mathcal H_\Phi$ with a latent representation $\Phi$. 
If a predictor is multicalibrated w.r.t $\mathcal H_{1,\Phi}, \mathcal H_{2,\Phi}$ respectively, then it is also multicalibrated w.r.t. $\mathcal H_\Phi$.
Therefore, we separately study the convergence with two grouping function classes. In Proposition~\ref{prop:app_converge_h1}, we show the convergence for a subset of $\mathcal H_{1,\Phi}$ consisting of covariate-based grouping functions, which is a corollary of \citeauthor{MCBoost}'s result.  
As a greater challenge, we derive convergence for $\mathcal H_{2,\Phi}$ when data follows multivariate normal distributions.
\begin{theorem}[Covergence for $\mathcal H_{2,\Phi}$ (informal, see Theorem~\ref{theo:app_converge_h2})]
\label{theo:converge_h2}
	Consider $X\in\mathbb R^d$ with $X=(\Phi,\Psi)^T$. Assume that $(\Phi,\Psi,Y)$ follows a multivariate normal distribution $\mathcal N_{d+1}(\mu,\Sigma)$ where the random variables are in general position such that $\Sigma$ is positive definite. For any distribution $D$ supported on $\mathcal X \times \mathcal Y$, take the grouping function class $\mathcal H=\{h\in\mathcal H_{2,\Phi}:h(x,y)=c_x^Tx+c_yy + c_b, c_x\in\mathbb R^d, c_y, c_b\in\mathbb R \}$ and the predictor class $\mathcal F=\mathbb R^{\mathcal X}$. For an initial predictor $f_0(x)=\mathbb E[Y|x]$, run $\text{MC-Pseudolabel}(D, \mathcal H, \mathcal F)$ without rounding, 
then $f_t(x) \rightarrow \mathbb E[Y|\Phi(x)]$ as $t \rightarrow \infty$ with a convergence rate of $\mathcal O(M(\Sigma)^t)$ where $0\leq M(\Sigma) <1$.
\end{theorem}

MC-Pseudolabel is also an optimization paradigm for invariance. 
Certified multicalibration in Theorem~\ref{theo:stopping_criteria} also implies certified invariance. Furthermore, MC-Pseudolabel introduces no extra hyperparameters to tradeoff between risks and robustness. 
Both certified invariance and light-weighted hyperparameters simplify model selection, which is challenging for out-distribution generalization because of unavailable validation data from target distributions~\citep{domainBed}.
MC-Pseudolabel has light-weighted optimization consisting of a series of supervised regression. It introduces an overhead to Empirical Risk Minimization by performing regression on level sets. However, the extra burden is linear regression by designing the grouping function class as linear space. Furthermore, regression on different level sets can be parallelized. Computational complexity is further analyzed in section~\ref{sec:app_complexity}. 

%% file: experiment.tex
\section{Experiments}
\subsection{Settings}
\label{subsec:exp_settings}
We benchmark MC-Pseudolabel on real-world \emph{regression} datasets with distributional shift. We adopt two experimental settings. For the \emph{multi-environment} setting, algorithms are provided with training data collected from multiple annotated environments. Thereafter, the trained model is assessed on new environments. For the \emph{single-environment} setting, algorithms are trained on a single source distribution. There could be latent sub-populations in training data, but environment annotations are unavailable. The trained model is assessed on a target dataset with distribution shift from the source. 

\textbf{Datasets}\quad We experiment on PovertyMap~\citep{poverty} and ACSIncome~\citep{income} for the multi-environment setting, and VesselPower~\citep{shift2} for the single-environment setting. As the only regression task in WILDS~\citep{wilds}, a popular benchmark for in-the-wild distribution shift, PovertyMap performs poverty index estimation for different spatial regions by satellite images. Data are collected from both urban and rural regions, by which the environment is annotated. The test dataset also covers both environments, but is collected from different countries. The primary metric is \emph{Worst-U/R Pearson}, the worst Pearson correlation of prediction between rural and urban regions. The other two datasets are tabular, where natural concept shift ($Y|X$ shift) is more common due to existence of missing variables and hidden confounders~\citep{whyshift}. ACSIncome~\citep{income} performs personal income prediction with data collected from US Census sources across different US states. The task is converted to binary classification by an income threshold, but we take raw data 
for regression. Environments are partitioned by different occupations with similar average income. VesselPower comes from Shifts~\citep{shift,shift2}, a benchmark focusing on regression tasks with real-world distributional shift. The objective is to predict power consumption of a merchant vessel given navigation and weather data. Data are sampled under different time and wind speeds, causing distribution shift between training and test data. 

\textbf{Baselines}\quad For the multi-environment setting, baselines include ERM (Empirical Risk Minimization); methods for invariance learning which mostly adopts multi-objective optimization: IRM~\citep{IRM}, MIP~\citep{MIP}, IB-IRM~\citep{IB-IRM}, CLOvE~\citep{clove}, MRI~\citep{MRI}, REX~\citep{REX}, Fishr~\citep{fishr}; an alignment-based method from domain generalization: IDGM~\citep{IDGM}; and Group DRO~\citep{groupDRO}. Notably, CLOvE learns a calibrated predictor simultaneously on all environments, but it is optimized by multi-objective learning with a differentiable regularizer for calibration. For the singe-environment setting, baselines include reweighting based techniques: CVaR~\citep{cvar}, JTT~\citep{JTT}, Tilted-ERM~\citep{TERM}; a Distributionally Robust Optimization method $\chi^2$-DRO~\citep{x2dro}; and a data augmentation method C-Mixup~\citep{cmixup}. Other methods are not included because of specification in classification~\citep{mixup,cnc} or exposure to target distribution data during training~\citep{simple_reweight, retrainLast}. For all experiments, we train an Oracle ERM with data sampled from target distribution. 

\textbf{Implementation} \quad
We implement the predictor with MLP for ACSIncome and VesselPower, and Resnet18-MS~\citep{resnet} for PovertyMap, following WILDS' default architecture. The grouping function class is implemented according to Equation~\ref{eq:h_irm} and Equation~\ref{eq:h_jtt} for the multi-environment and single-environment setting respectively.
We follow DomainBed's protocol~\citep{domainBed} for \emph{model selection}. Specifically, we randomly sample 20 sets of hyperparameters for each method, containing both the training hyperparameters and extra hyperparameters from the robust learning algorithm. We select the best model across hyperparameters based on three model selection criteria, including in-distribution validation on the average of training data, worst-environment validation with the worst performance across training environments, and oracle validation on target data. Oracle validation is not recommended by DomainBed, which suggests limited numbers of access to target data. The entire run is repeated with different seeds for three times to measure standard errors of performances. Specifically for PovertyMap, an OOD validation set is provided for oracle validation. And we perform 5-fold cross validation instead of three repeated experiments, following WILDS' setup.
\subsection{Results}
\begin{table}[]
\caption{Results on multi-environment datasets, evaluated on test data using three model selection criteria. ID: validation with averaged performance on training data. Worst: validation with the worst performance across training environments. Oracle: validation with performance on sampled test set.}
\label{table:multi-env}
\centering
\resizebox{0.9\columnwidth}{!}{%
\begin{tabular}{@{}lcccccc@{}}
\toprule
        & \multicolumn{3}{c}{ACSIncome: RMSE $\downarrow$}                                      & \multicolumn{3}{c}{PovertyMap: Worst-U/R Pearson $\uparrow$}             \\ \cmidrule(l){2-7} 
 Method             & ID                   & Worst               & Oracle               & ID                 & Worst     & Oracle    \\ \midrule
ERM            &  $0.487${\scriptsize$\pm0.009$}          & $0.487${\scriptsize$\pm0.009$}          & $0.452${\scriptsize$\pm0.012$}          & $0.48${\scriptsize$\pm0.06$}          & $0.48${\scriptsize$\pm0.06$} & $0.49${\scriptsize$\pm0.07$} \\
IRM            & $0.466${\scriptsize$\pm0.002$}          & $0.466${\scriptsize$\pm0.002$}          & $0.465${\scriptsize$\pm0.002$}          & $0.38${\scriptsize$\pm0.07$}          & $0.39${\scriptsize$\pm0.06$} & $0.45${\scriptsize$\pm0.08$} \\
MIP            & $0.457${\scriptsize$\pm0.008$}          & $0.454${\scriptsize$\pm0.012$}          & $0.454${\scriptsize$\pm0.012$}          & $0.40${\scriptsize$\pm0.09$}          & $0.39${\scriptsize$\pm0.10$} & $0.43${\scriptsize$\pm0.08$} \\
IB-IRM         & $0.463${\scriptsize$\pm0.003$}          & $0.463${\scriptsize$\pm0.003$}          & $0.438${\scriptsize$\pm0.009$}          & $0.39${\scriptsize$\pm0.07$}          & $0.37${\scriptsize$\pm0.05$} & $0.43${\scriptsize$\pm0.06$} \\
CLOvE          & $0.455${\scriptsize$\pm0.005$}          & $0.454${\scriptsize$\pm0.002$}          & $0.450${\scriptsize$\pm0.005$}          & $0.46${\scriptsize$\pm0.09$}          & $0.42${\scriptsize$\pm0.13$} & $\mathbf{0.51}${\scriptsize$\pm0.06$} \\
MRI            & $0.458${\scriptsize$\pm0.011$}          & $0.458${\scriptsize$\pm0.011$}          & $0.455${\scriptsize$\pm0.013$}          & $0.47${\scriptsize$\pm0.10$}          & $0.46${\scriptsize$\pm0.08$} & $0.49${\scriptsize$\pm0.07$} \\
REX            & $0.466${\scriptsize$\pm0.009$}          & $0.464${\scriptsize$\pm0.009$}          & $0.458${\scriptsize$\pm0.003$}          & $0.43${\scriptsize$\pm0.09$}          & $0.42${\scriptsize$\pm0.09$} & $0.45${\scriptsize$\pm0.09$} \\
Fishr          & $0.458${\scriptsize$\pm0.006$}          & $0.455${\scriptsize$\pm0.012$}          & $0.450${\scriptsize$\pm0.008$}          & $0.42${\scriptsize$\pm0.09$}          & $0.41${\scriptsize$\pm0.09$} & $0.43${\scriptsize$\pm0.08$} \\
IDGM           & $1.843${\scriptsize$\pm0.018$}          & $1.843${\scriptsize$\pm0.018$}          & $1.843${\scriptsize$\pm0.018$}          & $0.02${\scriptsize$\pm0.07$}          & $0.01${\scriptsize$\pm0.15$} & $0.13${\scriptsize$\pm0.14$} \\
GroupDRO       & $0.481${\scriptsize$\pm0.035$}          & $0.449${\scriptsize$\pm0.017$}          & $0.433${\scriptsize$\pm0.013$}          & $0.38${\scriptsize$\pm0.15$}          & $0.37${\scriptsize$\pm0.16$} & $0.42${\scriptsize$\pm0.12$} \\
MC-Pseudolabel & $\mathbf{0.428}${\scriptsize$\pm0.009$} & $\mathbf{0.425}${\scriptsize$\pm0.012$} & $\mathbf{0.411}${\scriptsize$\pm0.011$} & $\mathbf{0.50}${\scriptsize$\pm0.07$} & $\mathbf{0.50}${\scriptsize$\pm0.07$} & $\mathbf{0.51}${\scriptsize$\pm0.08$} \\ \midrule
Oracle ERM        &                      &                      & $0.332${\scriptsize$\pm0.001$}          &                    &           & $0.71${\scriptsize$\pm0.05$} \\ \bottomrule
\end{tabular}%
}
\end{table}
\begin{figure}[]
\vskip -0.16in
\begin{center}
\centerline{\includegraphics[width=0.9\columnwidth]{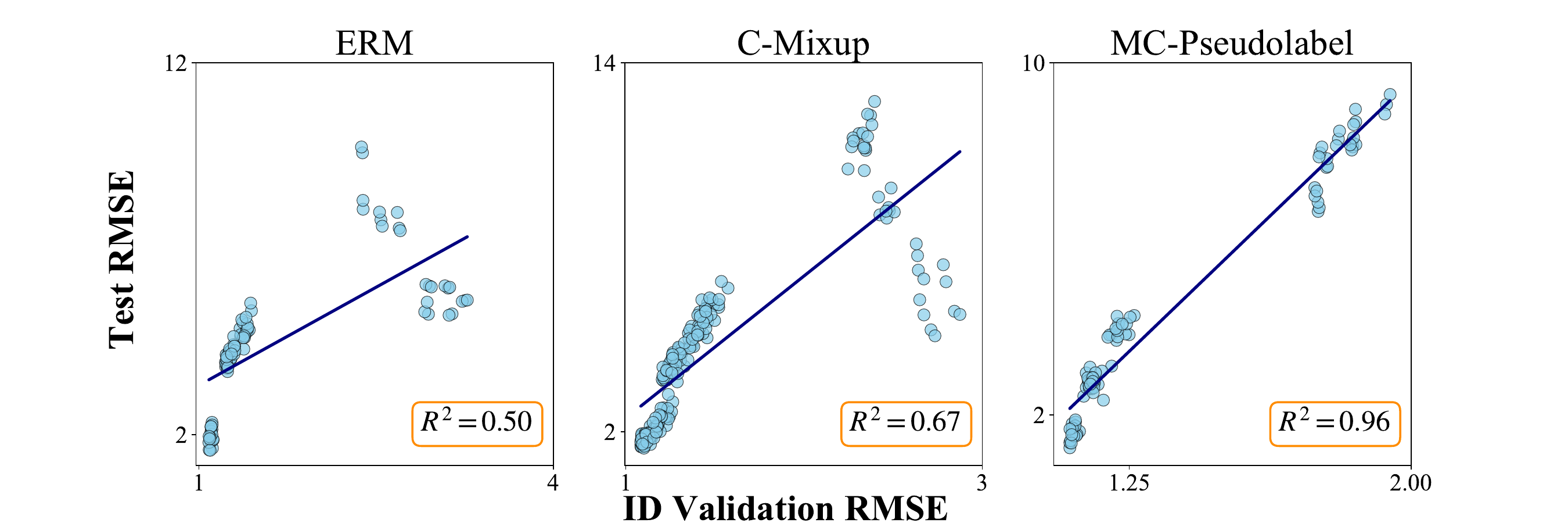}}
\label{fig:ontheline}
\caption{Accuracy-on-the-line beyond covariate shift: correlation between models' in-distribution and out-of-distribution risks on VesselPower.}
\end{center}
\vskip -0.4in
\end{figure}
Results are shown in Table~\ref{table:multi-env} for multi-environment settings and Table~\ref{table:vpower} for single-environment settings. 
MC-Pseudolabel achieves superior performance in all datasets with in-distribution and worst-environment validation which does not violate test data. For oracle validation, MC-Pseudolabel 
achieves  comparable performances to the best method. 
For example, CLOvE, which also learns invariance by calibration, achieves best performance under oracle validation in PovertyMap, but it sharply degrades when target validation data is unavailable. It's because CLOvE tunes its regularizer's coefficient to tradeoff with ERM risk, whose optimal value depends on the target distribution shift.
\begin{wraptable}{r}{0.38\textwidth}
\vskip 0in
\caption{Single-environment results.}
\label{table:vpower}
\begin{center}
\resizebox{0.38\columnwidth}{!}{%
\begin{tabular}{@{}lcc@{}}
\toprule
        & \multicolumn{2}{c}{VesselPower: RMSE $\downarrow$} \\ \cmidrule(l){2-3} 
        
Method     & ID        & Oracle       \\ \midrule
ERM            & $1.92${\scriptsize$\pm0.23$}          &   $1.86${\scriptsize$\pm0.19$}           \\
CVaR           & $1.69${\scriptsize$\pm0.18$}         &    $\mathbf{1.49}${\scriptsize$\pm0.10$}         \\
JTT            & $1.75${\scriptsize$\pm0.27$}         &    $1.58${\scriptsize$\pm0.15$}          \\
Tilted-ERM     & $1.72${\scriptsize$\pm0.21$}          &   $1.61${\scriptsize$\pm0.12$}           \\
$\chi^2$-DRO            & $1.69${\scriptsize$\pm0.20$}         & $1.56${\scriptsize$\pm0.06$}             \\
C-Mixup        & $1.72${\scriptsize$\pm0.15$}          &   $1.56${\scriptsize$\pm0.08$}          \\
MC-Pseudolabel & $\mathbf{1.61}${\scriptsize$\pm0.20$}          &   $1.52${\scriptsize$\pm0.16$}           \\ \midrule
Oracle ERM        &          &   $1.18${\scriptsize$\pm0.01$}           \\ \bottomrule
\end{tabular}%
}
\end{center}
\vskip -0.2in
\end{wraptable}
In contrast, MC-Pseudolabel exhibits an advantage with in-distribution model selection. This is further supported by Figure~\ref{fig:ontheline}, which shows that MC-Pseudolabel's out-of-distribution errors strongly correlates with in-distribution errors. The experiment spans across different hyperparameters and seeds with the same model architecture on VesselPower. The phenomenon, known as accuracy-on-the-line~\citep{accuracy_on_line}, is well known for a general class of models under covariate shift. However, \citet{whyshift} shows accuracy-on-the-line does not exist under concept shift, which is the case for ERM and C-Mixup. This introduces significant challenge for model selection. However, MC-Pseudolabel recovers the accuracy to the line.

%% file: conclusion.tex
\section{Conclusion}
To conclude, we establish a new optimization framework for out-of-distribution generalization through extended multicalibration with joint grouping functions. While the current algorithm focuses on regression, there is potential for future work to extend our approach to general forms of tasks, particularly in terms of classification.

%% file: appendix.tex
\section{Related Work}
\label{sec:app_related}
\textbf{Multicalibration}\;\;
Multicalibration is first proposed by \citet{multicalibration} with binary grouping functions. \citet{multiaccuracy} and \citet{omipredictor} extend the grouping functions to real-valued functions. \citet{MCBoost} shows that with a sufficiently rich class of real-valued grouping functions, multicalibration actually implies accuracy. \citeauthor{MCBoost} also provides a boosting algorithm for both regression and multicalibration. The connection between multicalibration and distribution shift is first studied by \citet{UniversalAdaptability}, who proves that $\ell_1$ multicalibration error remains under covariate shift, given a sufficiently large real-valued grouping function class. \citeauthor{UniversalAdaptability} further shows that under covariate shift, a multicalibrated predictor can perform statistical inference of the average outcome of a sample batch. In contrast, we derive a robustness result for individual prediction of outcomes for $\ell_2$ multicalibrated predictors. In addition, \citet{clove} studies the equivalence of Invariant Risk Minimization and simultaneous calibration on each environment. Our equivalence results for multicalibration can be perceived as a generalization of \citeauthor{clove}'s results beyond the multi-environment setting, by deriving a mapping between density ratios and grouping functions. We also extend the equivalence to approximately multicalibrated and approximately invariant predictors. Furthermore, we move beyond \citeauthor{clove}'s multi-objective optimization with Lagrangian regularization, by proposing a new post-processing optimization framework consisting of a series of supervised regression. Meanwhile, \citet{postprocess} discusses connections between calibration and post-processing, which is an equivalent expression of invariance. There are other extensions of multicalibration, such as \citet{happy} who generalize the term $Y-f(X)$ in multicalibration's definition to a class of general functions. While our work is the first to generalize the grouping functions $h$ to consider the outcomes. 

\textbf{Out-of-distribution Generalization Beyond Covariate Shift}\;\;
Despite abundant literature from domain generalization that focuses on image classification where covariate shift dominates, research on algorithmic robustness on regression tasks beyond covariate shift is relatively limited.
The setting can be categorized according to if the source distribution is partitioned into several environments. For the multi-environment generalization setting, Invariant Risk Minimization and its variants assume that outcomes are generated by a common causal structural equation across all environments, and aims to recover such an invariant (or causal) predictor~\citep{IB-IRM,IRM,MRI,MIP,REX,fishr}. Group DRO~\citep{groupDRO} is a simple but surprisingly strong technique that optimizes for the worst group risk with reweighting of environments. There are also meta-learning methods~\citep{MAML} that handles multi-environment generalization with bi-level optimization. For the single environment setting, Distributionally Robust Optimization optimizes for the worst-case risk in an uncertainty set of distributions centering around the source distribution~\citep{x2dro,fdro_opt,fdro,cvar,wdro}. Another branch of research is targeted at mitigating spurious correlation with an assumption of simplicity bias, which utilizes a simple model to discover latent sub-populations and then correct the biased predictor by sample reweighting~
\citep{TERM,focalLoss,JTT},  retraining on a subgroup-balanced dataset or a small batch from target distribution~\citep{simple_reweight,retrainLast,cnc}, or perform Invariant Risk Minimization on discovered subgroups~\citep{EIIL}. 
Data augmentation is a prevalent technique to enhance algorithmic robustness for vision tasks. Quite a lot of these methods are tailored for classification. For example, Mixup~\citep{mixup} interpolates between features of samples with the same label. The approach is extended to regression settings by C-Mixup~\citep{cmixup}. Pseudolabelling is a common technique for out-of-distribution generalization, but typically adopted in a setting  with exposure to unlabelled samples from target distribution, known as domain adaptation~\citep{gda}. However, MC-Pseudolabel generate pseudolabels for the source distribution itself. 

\section{Grouping Functions for Distributionally Robust Optimization}
\label{sec:app_grouping}
Distributionally Robust Optimization assumes the target distribution to reside in an uncertainty set $\mathcal{P}$ of distributions centering around the source distribution $P_S$. For example, \citet{alpha_dro} formulates the uncertainty set as arbitrary subgroups that has a proportion of at least $\alpha_0\in (0,1)$. \citeauthor{alpha_dro} only consider subgroups of covariates:
\begin{align}
    \mathcal P(X) = \{ P(X): &\text{ there exists a probability measure } P'(X),  \\
    &P_S(X) = \alpha P(X) + (1-\alpha)P'(X), \alpha \geq \alpha_0 \}.
\end{align}
By the correspondence between density ratios and grouping functions, the equivalent design of a grouping function class is given by:
\begin{align}
	\mathcal H=\left\{ h \in \mathbb R^{\mathcal X}: 0\leq h(x)\leq \frac{1}{\alpha_0}, \;\;\forall x  \right\}.
\end{align}
We can also extend the grouping functions to consider both covariates and outcomes, such that general subgroups are incorporated into the uncertainty set:
\begin{align}
	\mathcal H=\left\{ h \in \mathbb R^{\mathcal X\times\mathcal Y}: 0\leq h(x,y)\leq \frac{1}{\alpha_0}, \;\;\forall x,y  \right\}.
\end{align}
In the case of grouping functions defined on $x$ and $y$ jointly, the grouping function class is not closed under affine transformation and is not a linear space spanned by density ratios, which suggests that a perfectly multicalibrated solution might not exist.
However, approximately multicalibrated predictors can still be pursued.

\section{Model Discretization}
\label{app_discretize}
For continuous predictors, we take a preprocessing step to discretize the model to as many bins as possible such that the rounding error is negligible while still ensuring enough samples in individual bins. Specifically, we equally split the outcomes of predictors to bins with equal intervals from the minimum to maximum of model output. We start from a minimum bin number $m=10$, and keeps increasing $m$ as long as $90\%$ of the samples reside in a bin with at least 30 samples. When the criterion is violated, we stop increasing $m$ and select it as the final bin number. The model discretization procedure is fixed across all experiments.

\section{Computational Complexity}
\label{sec:app_complexity}
We assume that the predictor's outcomes are uniformly distributed. Denote the average bin size by $N_b$, which is a constant around 30 in our implementation. The bin number is given by $m=N/N_b$ where $N$ is the sample size. For neural networks, $N$ represents the batch size. The overhead of MC-Pseudolabel compared to Empirical Risk Minimization is linear regression on each bin, whose sample complexity is $\mathcal O(N_b^3)$ with OLS. Please note that an individual linear regression for around 30 samples is extremely cheap. A non-parallel implementation of regression on every bin scales linearly with the bin number $m$, so the overall complexity is $\mathcal O(N_b^2N)$. However, since the regression on each bin is independent, we adopt a multi-processing implementation. Denote the number of jobs by $J$, the overall time cost of MC-Pseudolabel is $\mathcal O(N_b^2N/J)$. As a comparison, OLS on $N$ samples has a computational complexity of $\mathcal O(N^3)$.

In conclusion, the complexity of MC-Pseudolabel scales linearly with sample size (or batch size for neural networks). Counterintuitively, increasing the bin number $m$ (and thus decreasing the bin size) actually decreases the computational complexity. This is because linear regression scales cubically with sample size, so decreasing the sample size in each bin is preferred to decreasing the bin number.

\section{Experiments}
\subsection{An Additional Experiment: Synthetic Dataset}
We start from a multi-environment synthetic dataset with a multivariate normal distribution corresponding to Theorem~\ref{theo:converge_h2}. In this experiment, we examine the optimization dynamics of MC-Pseudolabel. The data generation process is inspired by \citet{IRM}. The covariates can be sliced into $X=(S,V)^T$ with $S\in\mathbb R^9$ and $V\in\mathbb R$, where $S$ is the causal variable for $Y$ and $V$ is the spurious variable. The data is generated by the following structural equations:
\begin{align}
    S &\sim \mathcal N(0,1). \\
    Y &= \alpha_S^TS + \epsilon_Y,\;\;\;\; \alpha_S = (1,...,1)^T \in \mathbb R^9, \epsilon_Y \sim \mathcal N(0,0.5^2). \\
    V &= \alpha_V(\mathcal E)\cdot Y + \epsilon_V, \;\;\;\; \alpha_V(e_1)=1.25, \alpha_V(e_2)=0.75, \alpha_V(e_T)=-1,  \epsilon_V \sim \mathcal N(0,0.1^2).
\end{align}
The covariate $V$ is spuriously correlated with $Y$ because the coefficient $\alpha_V(e)$ depends on the specific environment $e$. We set $\alpha_V(\mathcal E)=1.25,0.75$ respectively for two training environments while $\alpha_V(\mathcal E)$ extrapolates to $-1$ during testing. A robust algorithm is supposed to bypass the spurious variable and output a predictor $f(X)=\alpha_S^TS$ in order to survive the test distribution where the correlation between $V$ and $Y$ is negated.

The predictor class for this dataset is linear models, and the environment classifier is implemented by MLP with a single hidden layer. In this experiment, we fix the training hyperparameters for the base linear model, and perform grid search over the extra hyperparameter introduced by robust learning methods. Baselines except for ERM and Group DRO share a hyperparameter which is the regularizer's coefficient, and Group DRO introduces a temperature hyperparameter. We search over their hyperparameter space and report RMSE metric on the test set in Figure~\ref{fig:simulation}. Most baselines exhibit a U-turn with an increasing hyperparameter, and the minimum point varies across methods. The sensitivity of hyperparameters implies the dependence on a strong model selection criterion, such as oracle model selection on target distribution. However, the dashed line for MC-Pseudolabel's error is tangent to all the U-turns of baselines, indicating a competitive performance of MC-Pseudolabel both with and without oracle model selection.
\begin{figure}[ht]
\begin{center}
\centerline{\includegraphics[width=0.8\columnwidth]{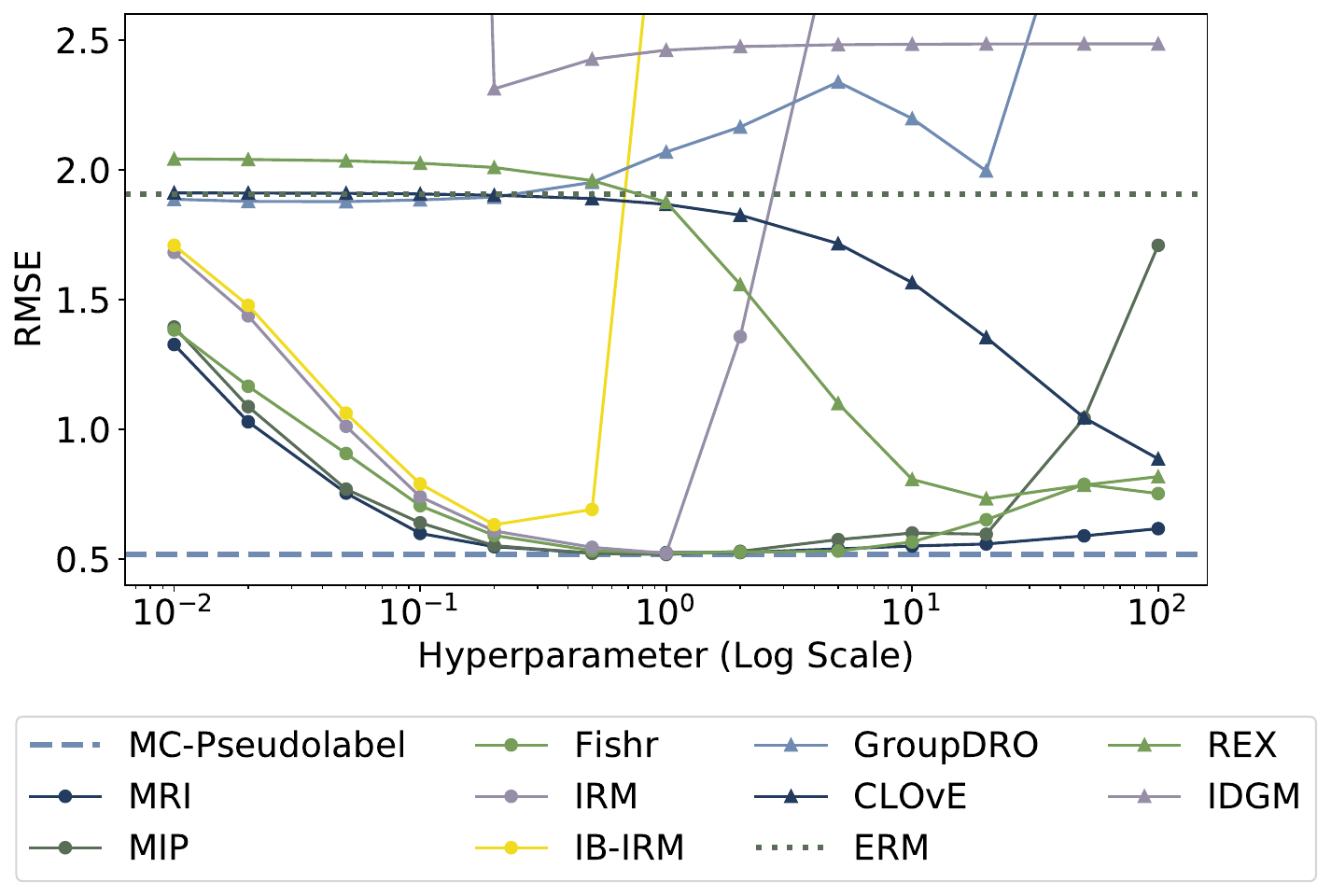}}
\caption{Results (RMSE) on the synthetic dataset. Curves show method performances across hyperparameters. Methods without extra hyperparameters are marked with dotted lines.}
\label{fig:simulation}
\end{center}
\end{figure}

We also investigate the evolution of pseudo labels $\tilde f_t$ in Algorithm~\ref{alg} to recover the dynamics of MC-Pseudolabel. The first row of Figure~\ref{fig:dynamics} demonstrates how pseudolabelling results in a multicalibrated predictor. It shows that pseudolabels for two environments deviate from model prediction at Step 0, but the gap quickly converges at Step 4, implying multicalibrated prediction. The second row provides insight about how pseudolabelling contributes to an invariant predictor. We observe that the curve of two environments are gradually merging because the pseudolabel introduces a special noise to the original label such that the correlation between the pseudolabel and spurious variable $V$ is weakened. As a result, the predictor will depend on the causal variable $S$ which is relatively more strongly correlated with pseudolabels.
\begin{figure}[ht]
\begin{center}
\centerline{\includegraphics[width=0.9\columnwidth]{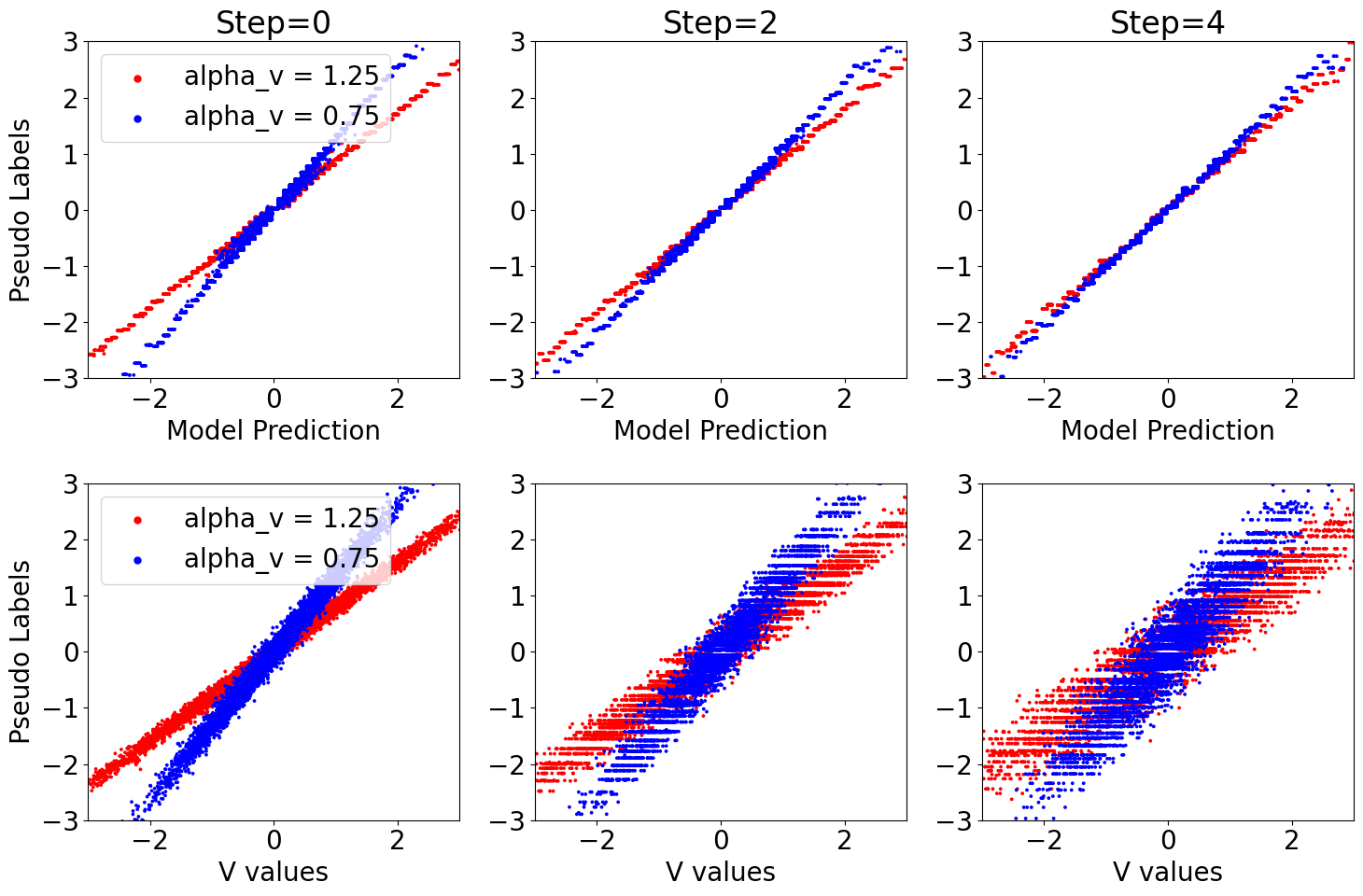}}
\caption{Evolution of pseudolabels during MC-Pseudolabel. The first row plots values of pseudolabels against model predictions. The second row plots values of pseudolabels against $V$. Columns represent different snapshots during optimization.}
\label{fig:dynamics}
\end{center}
\end{figure}

\subsection{VesselPower}
In figure~\ref{fig:app_ontheline}, we provide the correlation between models' in-distribution validation performance and out-of-distribution test performance across all methods.
\begin{figure}[ht]
\begin{center}
\centerline{\includegraphics[width=0.9\columnwidth]{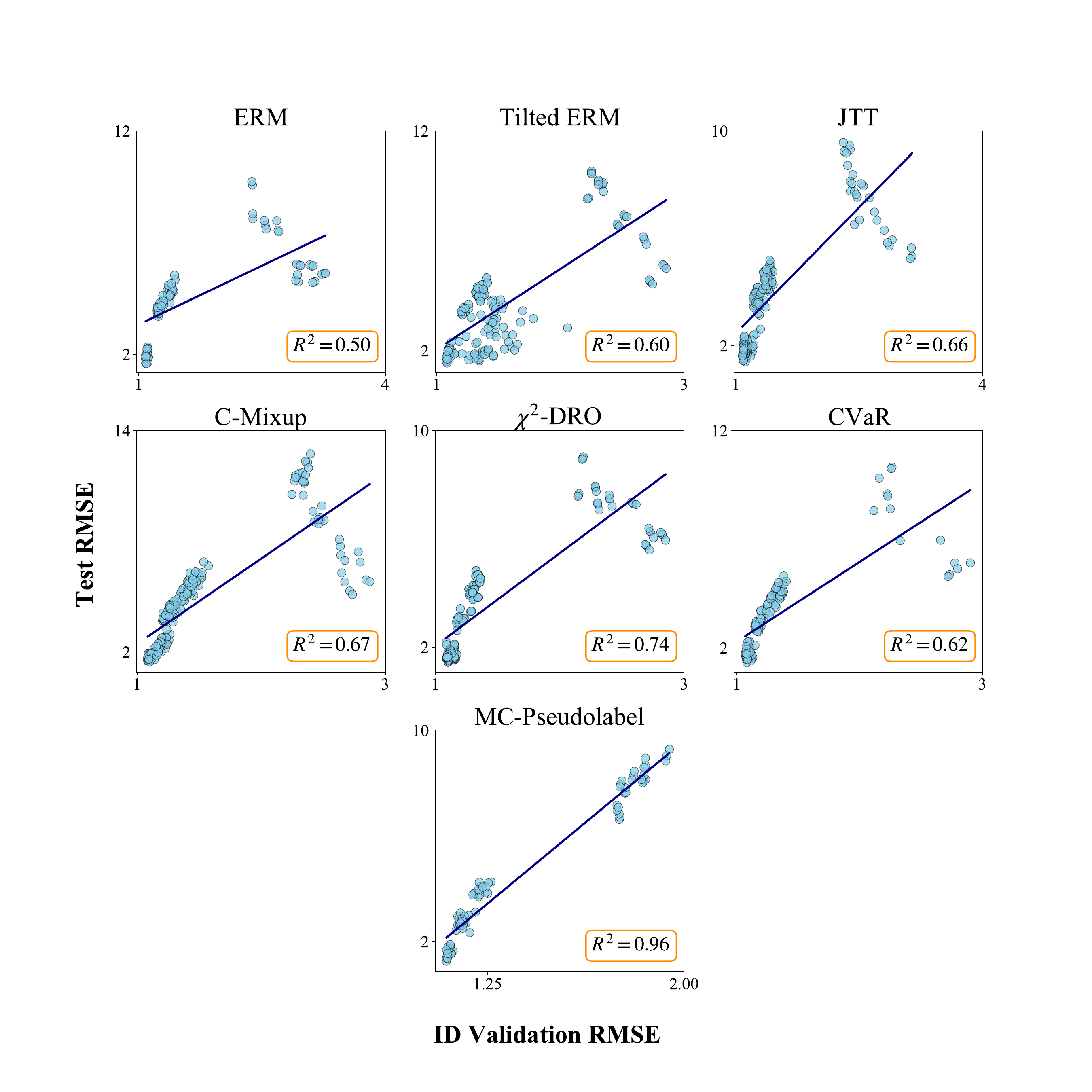}}
\vskip -0.2in
\caption{Correlation between models' in-distribution and out-of-distribution risks on VesselPower.}
\label{fig:app_ontheline}
\end{center}
\end{figure}

\subsection{Training Details}
\label{subsec:app_train_detail}
We follow DomainBed's protocol~\citep{domainBed} for model selection. Specifically, we randomly sample 20 sets of hyperparameters for each method, containing both the training hyperparameters of base models in Table~\ref{table:app_hyper_model} and extra hyperparameters from the robust learning algorithm in Table~\ref{table:app_hyper_robust}. We select the best model across hyperparameters based on three model selection criteria. In-distribution (ID) validation selects the model with the best metric on the average of an in-distribution validation dataset, which is sampled from the same distribution as the training data. Worst-environment (Worst) validation selects the best model by the worst performance across all environments in the in-distribution validation dataset. Worst validation is applicable only to the multi-environment setting. Oracle validation selects the best model by an out-of-distribution validation dataset sampled from the target distribution of test data. Oracle validation leaks the test distribution, so it is not recommended by DomainBed. However, most robust learning methods relies on out-of-distribution validation, so Domainbed suggests limited numbers of access to target data when using Oracle validation. Though MC-Pseudolabel already performs well under ID and Worst validation, we still report its performance under Oracle validation to compare the limit of robust learning methods regardless of model selection. Specifically for PovertyMap, an OOD validation set is provided for Oracle validation. The results of Oracle validation recover PovertyMap's public benchmark. 

Following DomainBed, the entire model selection procedure is repeated with different seeds for three times to measure standard errors of performances. Thus, we have totally 60 runs per method per dataset. Specifically for PovertyMap, we follow WILDS' setup~\citep{wilds} and perform 5-fold cross validation instead of three repeated experiments. For each fold of the dataset, we conduct the model selection procedure with a single seed, and we report the average and standard error of performances across 5 folds. Thus, the standard error measures both the difficulty disparity across folds and the model's instability.

The grouping function class of MC-Pseudolabel is implemented according to Equation~\ref{eq:h_irm} and Equation~\ref{eq:h_jtt} for the multi-environment and single-environment setting respectively. For the multi-environment setting, the environment classifier $p(e|x,y)$ is implemented as MLP with a single hidden layer of size 100 for tabular datasets including Simulation and ACSIncome. For PovertyMap, the environment classifier is implemented by Resnet18-MS with the same architecture as the predictor. For the single-environment setting of VesselPower, the identification model $f_{id}$ is implemented as a Ridge regression model.

\begin{table}[ht]
\caption{Hyperparameters for model architecture.}
\label{table:app_hyper_model}
\centering
\resizebox{\columnwidth}{!}{%
\begin{tabular}{@{}lcccc@{}}
\toprule
                        & Simulation & ACSIncome & VesselPower & PovertyMap  \\ \midrule
Architecture            & Linear     & MLP       & MLP         & Resnet18-MS \\
Hidden Layer Dimensions & None       & 16  & 32, 8  & Standard~\citep{wilds}    \\
Optimizer               & Adam       & Adam      & Adam        & Adam        \\
Weight Decay               & 0       & 0      & 0        & 0        \\
Loss                    & MSE        & MSE       & MSE         & MSE         \\ 
Learning Rate                    & 0.1        & $10^{\text{Uniform}[-3, -1]}$       & $10^{\text{Uniform}[-3, -1]}$         & $10^{\text{Uniform}[-4,-2]}$         \\
Batch Size                    & 1024        & [256, 512, 1024, 2048]       &  [256, 512, 1024, 2048]         &  [32, 64, 128]         \\ 
\bottomrule
\end{tabular}%
}
\end{table}
\begin{table}[ht]
\caption{Hyperparameters for robust learning methods.}
\label{table:app_hyper_robust}
\centering
\begin{tabular}{@{}lc@{}}
\toprule
                        & Range \\ \midrule
Regularizer Coefficient & $10^{\text{Uniform}[-3, 2]}$      \\
$\eta$ (GroupDRO) & $10^{\text{Uniform}[-3, 2]}$      \\
$\alpha$ (JTT)                    & [0.1, 0.2, 0.5, 0.7]      \\
$\lambda$ (JTT)                   &  [5, 10, 20, 50]     \\ 
$\eta$ ($\chi^2$-DRO)                   &  [0.2, 0.5, 1, 1.5]     \\ 
$t$ (Tilted-ERM)        &  [0.1, 0.5, 1, 5, 10, 50, 100, 200] \\
$\alpha$ (C-Mixup)                   &  [0.5, 1, 1.5, 2]     \\ 
$\sigma$ (C-Mixup)                   &  [0.01, 0.1, 1, 10, 100]     \\ 
\bottomrule
\end{tabular}
\end{table}

\subsection{Software and Hardware}
\label{subsec:app_hardware}
Our experiments are based on the architecture of PyTorch~\citep{pytorch}. 
Each experiment with a single set of hyperparameters is run on one NVIDIA GeForce RTX 3090 with 24GB of memory, taking at most 15 minutes.

\section{Theory}
\label{sec:app_theory}
\subsection{Multicalibration and Bayes Optimality under Covariate Shift}

\begin{assumption}[Restatement of Assumption~\ref{assum:sufficient_grouping}]
\label{app_assum:app_sufficient_grouping}
\;
\vskip 0.01in
1. (Closure under Covariate Shift)
For a set of probability measures $\mathcal P(X)$ containing the source measure $P_S(X)$, 
\begin{equation}
\label{eq:densityratio_hx}
	\forall P \in \mathcal P, h\in \mathcal H \Rightarrow \frac{p(\cdot)}{p_S(\cdot)}\cdot h(\cdot) \in \mathcal H.
\end{equation}
2. (($\gamma, \rho$)-Weak Learning Condition)
For any $P \in \mathcal P(X)P_S(Y|X) \equiv \{P'(X) P_S(Y\mid X) : P' \in \mathcal{P}\}$ with the source conditional measure $P_S(Y|X)$ and every measurable set $G \subset \mathcal X$ satisfying $P(X \in G) > \rho$, if
\begin{align}
	\mathbb E_P[(\mathbb E_P[Y|X]-Y)^2|X \in G]  
	< \mathbb E_P[(\mathbb E_P[Y|X \in G]-Y)^2|X \in G] - \gamma,
\end{align}
then there exists $h \in \mathcal H$ satisfying
\begin{align}
	\mathbb E_P[(h(X)-Y)^2|X \in G]  
	< \mathbb E_P[(\mathbb E_P[Y|X \in G]-Y)^2|X \in G] - \gamma.
\end{align}
\end{assumption}

\begin{lemma}[\citet{MCBoost}]
\label{lem:app_mc_acc}
	Fix any distribution $P\in\mathcal P(X,Y)$, any model $f : \mathcal{X} \rightarrow [0,1]$, and any class of real valued functions $\mathcal{H}$ that is closed under affine transformation. Let:
\[
\mathcal{H}_1 = \{h \in \mathcal{H} : \max_{x\in\mathcal{X}}h(x)^2 \leq 1\}
\]
be the set of functions in $\mathcal{H}$ upper-bounded by 1 on $\mathcal{X}$. Let $m = |\text{Range}(f)|, \gamma > 0$, and $\alpha \leq \frac{\gamma^3}{16m}$. Then if $\mathcal{H}$ satisfies the $(\gamma, \frac{\gamma}{m})$-weak learning condition and $f$ is $\alpha$-approximately $\ell_2$ multicalibrated with respect to $\mathcal{H}_1$ and $P$, then $f$ has squared error
\[
\mathbb{E}_P [(f(x) - y)^2] \leq \mathbb{E}_P [(f^*(x) - y)^2] + 3\gamma,
\]
where $f^*(x)=\mathbb E_{P}[Y|x]$.
\end{lemma}

\begin{definition}
	For a probability measure $P(X,Y)$ and a predictor $f:\mathcal X\rightarrow [0,1]$, let $\mathcal H \subset \mathbb R^{\mathcal X\times\mathcal Y}$ be a function class. We say that $f$ is $\alpha$-approximately $\ell_1$ multicalibrated w.r.t. $\mathcal H$ and $P$ if for all $h \in \mathcal H$:
\begin{align}
	&\;\;\;\;\;K_1(f, h, P) \\ 
	&= \int \Big|\mathbb E_P\left[h(X,Y)(Y-v)\big|f(X)=v\right]\Big| dP_{f(X)}(v) \\
	&\leq \alpha.
\end{align}
\end{definition}

\begin{lemma}[~\citet{uncertain}]
\label{lem:app_uniadapt}
	Suppose $P_S,P_T \in \mathcal P$ have the same conditional label distribution, and suppose $f$ is $\alpha$-approximately $\ell_1$ multicalibrated with respect to $P_S$ and $\mathcal{H}$. If $\mathcal{H}$ satisfies Equation~\ref{eq:densityratio_hx}, then $f$ is also $\alpha$-approximately $\ell_1$ multicalibrated with respect to $P_T$ and $\mathcal{H}$:
\end{lemma}

\begin{lemma}
\label{lem:app_k1k2}
	For a predictor $f:\mathcal X \rightarrow [0,1]$ and a grouping function $h$ satisfying $\max_{x\in\mathcal X,y\in\mathcal Y}h(x,y)^2\leq B$ where $B>0$,  
\begin{align}
	\frac{1}{\sqrt B}K_2(f, h,P) \leq K_1(f, h,P) \leq \sqrt{K_2(f, h,P)}.
\end{align} 
\end{lemma}
\begin{remark}
	The lemma is extended from \citet{uncertain}'s result for $B=1$.
\end{remark}
\begin{proof}
First we prove $K_2(f, h,P) \leq \sqrt B K_1(f, h,P)$. For any $v\in [0,1]$,
\begin{align}
	&\;\;\;\;\Big(\mathbb E_P\left[h(X,Y)(Y-v)\big|f(X)=v\right]\Big)^2 \\
	&= \Big|\mathbb E_P\left[h(X)(Y-v)\big|f(X)=v\right]\Big| \cdot
	\Big|\mathbb E_P\left[h(X,Y)(Y-v)\big|f(X)=v\right]\Big| \\
	\label{eq:app_lemk1k2_1}
	&\leq  \sqrt{\mathbb E\left[h(X,Y)^2\big|f(X)=v\right] 
	\mathbb E\left[(Y-v)^2\big|f(X)=v\right]} \\
	&\;\;\;\;\cdot
	\Big|\mathbb E_P\left[h(X,Y)(Y-v)\big|f(X)=v\right]\Big| \\
	&\leq \sqrt B \cdot \Big|\mathbb E_P\left[h(X,Y)(Y-v)\big|f(X)=v\right]\Big|.
\end{align}
Equation~\ref{eq:app_lemk1k2_1} follows from the Cauchy-Schwarz inequality. 
\begin{align}
	K_2(f, h,P) &=  \underset{v\sim P_{f(X)}}{\mathbb E} \left[ \Big(\mathbb E_P\left[h(X,Y)(Y-v)\big|f(X)=v\right] \Big)^2 \right]  \\
	&\leq \sqrt B \underset{v\sim P_{f(X)}}{\mathbb E} 
	\Big|\mathbb E_P\left[h(X,Y)(Y-v)\big|f(X)=v\right]\Big| \\
	&= \sqrt B K_1(f, h,P).
\end{align}
Then we prove $K_1(f, h,P) \leq \sqrt{BK_2(f, h,P)}$. Still from the Cauchy-Schwarz inequality:
\begin{align}
	K_1(f, h,P) &=  \underset{v\sim P_{f(X)}}{\mathbb E} 
	\Big|\mathbb E_P\left[h(X,Y)(Y-v)\big|f(X)=v\right]\Big| \\
	&\leq \sqrt {\underset{v\sim P_{f(X)}}{\mathbb E} [1^2]
	 \underset{v\sim P_{f(X)}}{\mathbb E} \left[\Big(\mathbb E_P\left[h(X,Y)(Y-v)\big|f(X)=v\right] \Big)^2\right]}  \\
	 &= \sqrt {K_2(f, h,P)}.
\end{align}
\end{proof}

\begin{theorem}[Restatement of Theorem~\ref{theo:regret_covariate}]
\label{theo:app_regret_covariate}
	For a source measure $P_S(X,Y)$ and a set of probability measures $\mathcal P(X)$ containing $P_S(X)$, given a predictor $f:\mathcal X \rightarrow [0,1]$ with finite range $m:=|\text{Range}(f)|$, consider a grouping function class $\mathcal H$ closed under affine transformation satisfying Assumption \ref{assum:sufficient_grouping} with $\rho = \gamma/m$. If $f$ is $\frac{\gamma^6}{256m^2}$-approximately $\ell_2$ multicalibrated w.r.t $P_S$ and $\mathcal H$'s bounded subset $\mathcal H_1 := \left\{ h \in \mathcal H:\max_{x \in \mathcal X}h(x)^2 \leq 1 \right\}$, then for any target measure $P_T \in \mathcal P(X)P_S(Y|X)$,
\begin{align}
	R_{P_T}(f) \leq R_{P_T}(f^*) + 3\gamma,
\end{align}
where $f^*(x)=\mathbb E_{P_T}[Y|x]$ is the optimal regression function in each target distribution.
\end{theorem}
\begin{proof}
	For $|h(x)|\leq 1$ and $f\in[0,1]$, according to Lemma~\ref{lem:app_k1k2},  
\begin{align}
	K_2(f, h,P) \leq K_1(f, h,P) \leq \sqrt{K_2(f, h,P)}.
\end{align}
Since $K_2(f, h,P_S)\leq \frac{\gamma^6}{256m^2}$ for any $h\in H_1$, 
\begin{align}
	K_1(f, h,P_S)\leq \frac{\gamma^3}{16m}.
\end{align}
With Lemma~\ref{lem:app_uniadapt}, for any $h\in H_1$ and $P_T \in\mathcal P$, we have
\begin{align}
	K_1(f, h,P_T)\leq \frac{\gamma^3}{16m}.
\end{align}
Again, it follows from Lemma~\ref{lem:app_k1k2} that:
\begin{align}
	K_2(f, h,P_T)\leq \frac{\gamma^3}{16m}.
\end{align}
With Lemma~\ref{lem:app_mc_acc}, we have
\begin{align}
	\mathbb{E}_{P_T} [(f(x) - y)^2] \leq \mathbb{E}_{P_T} [(f^*(x) - y)^2] + 3\gamma,
\end{align}
which completes the proof.
\end{proof}

\subsection{Multicalibration and Invariance under Concept Shift}

\begin{theorem}[Restatement of Theorem~\ref{theo:mc_inv}]
\label{theo:app_approx_mc_inv}
	For a set of absolutely continuous probability measures $\mathcal P(X,Y)$ containing the source measure $P_S(X,Y)$, consider a predictor $f:\mathcal X\rightarrow [0,1]$. Assume the grouping function class $\mathcal H$ satisfies the following condition:
\begin{align}
	\label{eq:app_density_ratio}
	\mathcal H \supset \left\{ h(x,y)=\frac{p(x,y)}{p_S(x,y)} \Big| P \in \mathcal P(X,Y) \right\}.
\end{align}
If $f$ is $\alpha$-approximately $\ell_2$ multicalibrated w.r.t. $\mathcal H$ and $P_S$, then for any measure $P \in \mathcal P(X,Y)$,
\begin{align}
	R_P(f) \leq \inf_{g:[0,1]\rightarrow [0,1]} R_P(g\circ f) + 2\sqrt \alpha.
\end{align}
\end{theorem}
\begin{proof}
	For any $h(x,y)=p(x,y)/p_S(x,y)$ where $P \in\mathcal P$, since $f$ is $\alpha$-approximately $\ell_2$ multicalibrated, $K_2(f,h,P_S)\leq\alpha$. It follows from Lemma~\ref{lem:app_k1k2} that $K_1(f,h,P_S)\leq\sqrt\alpha$.
\begin{align}
	K_1(f,h,P_S) &= \int \Big|\mathbb E_{P_S}\left[h(X,Y)(Y-v)\big|f(X)=v\right]\Big| dP_S(f^{-1}(v)) \\
	&= \int \left|\int \frac{p(x,y)}{p_S(x,y)}(y-v) p_S(x,y\big|f=v)d(x,y) \right| dP_S(f^{-1}(v)) \\
	&= \int \left| \int \frac{dP(f^{-1}(v))}{dP_S(f^{-1}(v))}(y-v) p(x,y\big|f=v)d(x,y) \right| dP_S(f^{-1}(v)) \\
	&= \int  \frac{dP(f^{-1}(v))}{dP_S(f^{-1}(v))} \left| \int (y-v) p(x,y\big|f=v)d(x,y) \right| dP_S(f^{-1}(v)) \\
	&= \int   \Big| \mathbb E_{P}\left[Y-v\big|f(X)=v\right] \Big| dP(f^{-1}(v)) \\
	\label{eq:app_approx_mc_inv_2}
	&= K_1(f,1,P).
\end{align}
Thus, we have $K_1(f,1,P)\leq \sqrt \alpha$ for any $P\in\mathcal P$. We will prove an equivalent form of $\ell_1$ calibration error:
\begin{align}
\label{eq:app_approx_mc_inv_1}  
	K_1(f,1,P) &= \sup_{\eta:[0,1]\rightarrow[-1,1]} \int \eta(v)\mathbb E_P\left[Y-v\big|f(X)=v \right]dP_{f(X)}(v) \\
	&= \sup_{\eta:[0,1]\rightarrow[-1,1]} \mathbb E_P[\eta(f(X))(Y-f(X))].
\end{align}
$K_1(f,1,P) \leq \sup_{\eta:[0,1]\rightarrow[-1,1]} \mathbb E_P[\eta(f(X))(Y-f(X))]$ can be proved by taking $\eta(v) = 2\mathbb I[v>0] - 1$. On the other hand,
\begin{align}
	\int \eta(v)\mathbb E_P\left[Y-v\big|f(X)=v \right]dP_{f(X)}(v) 
	&\leq \int  \Big|\eta(v)\Big| \cdot \Big|\mathbb E_P\left[Y-v\big|f(X)=v \right] \Big| dP_{f(X)}(v) \\
	&\leq K_1(f,1,P).
\end{align}
Actually the right hand side of Equation~\ref{eq:app_approx_mc_inv_1} resembles smooth calibration~\citep{measureOfCalibration}, which restricts $\eta$ to Lipschitz functions. Based on smooth calibration, \citeauthor{postprocess}~\citep{postprocess} shows that approximately calibrated predictors cannot be improved much by post-processing. In the above we present a similar proof for $\ell_1$ calibration error.

For any $g:[0,1]\rightarrow [0,1]$, there exists $\eta:[0,1]\rightarrow[-1,1]$ such that $g(v)=v+\eta(v)$ for any $v\in[0,1]$. 
\begin{align}
	R_P(g\circ f) &= \mathbb E_P\left[\big(Y-f(X)-\eta(f(X))\big)^2\right] \\
	&= \mathbb E_P\left[(Y-f(X))^2\right] - 2 \mathbb E_P\left[(Y-f(X)\eta(f(X))\right] + \mathbb E_P \left[ \eta(f(X))^2 \right] \\
	&\geq R_P(f) -\sup_{\eta':[0,1]\rightarrow[-1,1]} 2\mathbb E[\eta'(f(X))(Y-f(X))] \\
	&= R_P(f) - 2K_1(f,1,P) \\
	&\geq R_P(f) - 2\sqrt \alpha.
\end{align}
\end{proof}

\begin{theorem}[Restatement of Theorem~\ref{theo:equiv_mc_inv}]
	Assume samples are drawn from an environment $e\in\mathscr E$ with a prior $P_S(e)$ such that $\sum_{e\in\mathscr E}P_S(e)=1$ and $P_S(e)>0$.  The overall population satisfies $P_S(X,Y)=\sum_{e\in\mathscr E}P_e(X,Y)P_S(e)$ where $P_e(X,Y)$ is the environment-specific absolutely continuous measure. For a representation $\Phi \in \sigma(X)$ over features, define a function class $\mathcal H$ as:
\begin{align}
	\mathcal H:=\left\{ h(x,y)=\frac{p_e(x,y)}{p_S(x,y)} \Big| e \in \mathscr E \right\}.
\end{align}
1. If there exists a bijection $g^*:\text{supp}(\Phi) \rightarrow [0,1]$ such that $g^* \circ \Phi$ is $\alpha$-approximately $\ell_2$ multicalibrated w.r.t. $\mathcal H$ and $P_S$, then for any $e \in \mathscr E$,
\begin{align}
\label{eq:app_theo_mc_inv_equiv_1}
	R_{P_e}(g^* \circ \Phi) \leq \inf_{g:\text{supp}(\Phi)\rightarrow [0,1]} R_{P_e}(g\circ \Phi) + 2\sqrt \alpha.
\end{align} 

2. For $\Phi \in \sigma(X)$, if there exists $g^*:\text{supp}(\Phi) \rightarrow [0,1]$ such that  Equation~\ref{eq:app_theo_mc_inv_equiv_1} is satisfied for any $e \in \mathscr E$, then $g^* \circ \Phi$ is $\sqrt {2/D} \alpha^{1/4}$-approximately $\ell_2$ multicalibrated w.r.t. $\mathcal H$ and $P_S$, where $D=\min_{e\in \mathscr E}P_S(e)$.
\end{theorem}
\begin{proof} 
We first prove statement 1.

For any $g:\text{supp}(\Phi)\rightarrow [0,1]$, since $g^*$ is a bijection, $g \circ \Phi = (g\circ {g^*}^{-1})(g^* \circ \Phi)$ where $g\circ {g^*}^{-1} \in [0,1]^{[0,1]}$. Since $g^* \circ \Phi$ is $\alpha$-approximately $\ell_2$ multicalibrated, it follows from Theorem~\ref{theo:app_approx_mc_inv} that for any $e \in \mathscr E$,
\begin{align}
	R_{P_e}(g^* \circ \Phi) &\leq  R_{P_e}\big((g\circ {g^*}^{-1})(g^* \circ \Phi)\big) + 2\sqrt \alpha \\
	&= R_{P_e}(g \circ \Phi) + 2\sqrt \alpha.
\end{align}
Then we give a proof of statement 2, which is inspired by \citeauthor{postprocess}~\citep{postprocess}.

For simplicity let $f^* := g^* \circ \Phi$.
For any $e\in\mathscr E$ and any $\eta:[0,1]\rightarrow[-1,1]$, define $\beta:=E_{P_e}\left[(Y-f^*(X))\eta(f^*(X))\right]\in[-1,1]$.
Construct $\kappa(v):=\text{proj}_{[0,1]}(v+\beta\eta(v))$, where $\text{proj}_{[0,1]}(\cdot)=\max\{0,\min\{1,\cdot\}\}$. 
\begin{align}
	R_{P_e}(\kappa(f^*)) &=  \mathbb E_{P_e}\left[ \big(Y-\kappa(f^*(X))\big)^2 \right] \\
	&\leq \mathbb E_{P_e}\left[ \big(Y-f^*(X)-\beta\eta(f^*(X))\big)^2 \right] \\
	&= \mathbb E_{P_e}\left[ (Y-f^*(X))^2 \right] -2\beta^2 + \beta^2\mathbb E_{P_e}[\eta(f^*(x))^2] \\
	&\leq R_{P_e}(f^*) - 2\beta^2 + \beta^2 \\
	&= R_{P_e}(f^*) - \beta^2.
\end{align}
Rearranging the inequality above gives:
\begin{align}
\label{eq:app_theo_mc_inv_equiv_2}
	\Big(E_{P_e}\left[(Y-f^*(X))\eta(f^*(X))\right]\Big)^2 = \beta^2 \leq R_{P_e}(f^*) - R_{P_e}(\kappa(f^*)).
\end{align}
Since $\kappa\circ g^* \in \text{supp}(\Phi) \rightarrow [0,1]$, it follows from Equation~\ref{eq:app_theo_mc_inv_equiv_1} that:
\begin{align}
\label{eq:app_theo_mc_inv_equiv_3}
	R_{P_e}(f^*) - R_{P_e}(\kappa(f^*)) = R_{P_e}(g^* \circ \Phi) -  R_{P_e}((\kappa\circ g^*)\circ\Phi) \leq 2\sqrt \alpha.
\end{align}
Combining Equation~\ref{eq:app_theo_mc_inv_equiv_2} and Equation~\ref{eq:app_theo_mc_inv_equiv_3} gives $E_{P_e}\left[(Y-f^*(X))\eta(f^*(X))\right] \leq \sqrt 2 \alpha^{1/4}$ for any $\eta:[0,1]\rightarrow[-1,1]$.
From Equation~\ref{eq:app_approx_mc_inv_1}, it follows that:
\begin{align}
	K_1(f^*,1,P_e) = \sup_{\eta:[0,1]\rightarrow[-1,1]} \mathbb E_{P_e}[\eta(f^*(X))(Y-f^*(X))] \leq \sqrt 2 \alpha^{1/4}.
\end{align}
From Equation~\ref{eq:app_approx_mc_inv_2}, $K_1(f^*,h,P_S)\leq\sqrt 2 \alpha^{1/4}$ for any $h\in\mathcal H$. Further, for any $h\in\mathcal H$,
\begin{align}
	h(x,y) &= \frac{p_e(x,y)}{p_S(x,y)} \\
	&= \frac{p_e(x,y)}{\sum_{e'\in\mathscr E}{p_{e'}(x,y)}P_S(e')} \\
	&\leq \frac{p_e(x,y)}{p_{e}(x,y)P_S(e)} \\
	&\leq \frac{1}{D}.
\end{align}
By Lemma~\ref{lem:app_k1k2}, it follows that $K_2(f^*,h,P_S)\leq \sqrt {1/D}\cdot K_1(f^*,h,P_S) \leq \sqrt {2/D} \alpha^{1/4}.$
\end{proof}

\subsection{Structure of Grouping Function Classes}
In this subsection, we focus on Euclidean space where $\mathcal X \subset \mathbb R^d$ is compact and measurable for some $d\in Z^+$ and $\mathcal Y=[0,1]$. Grouping functions are assumed to be continuous, i.e., $h\in C(\mathcal X \times \mathcal Y)$. We consider absolutely continuous probability measures with continuous density functions.

\begin{proposition}[Restatement of Proposition~\ref{prop:linear_space}]
	Consider an absolutely continuous probability measure $P_S(X,Y)$ and a predictor $f:\mathcal X\rightarrow[0,1]$, define the maximal grouping function class that $f$ is multicalibrated with respect to:
	\begin{align}
		\mathcal H_f := \{ h\in C(\mathcal X \times\mathcal Y) : K_2(f,h,P_S)=0 \}.
	\end{align}
	Then $\mathcal H_f$ is a linear space.
	
	Particularly for $f_{\Phi}(x)=\mathbb E[Y|\Phi(x)]$ where $\Phi:\mathbb R^{d} \rightarrow \mathbb R^{d_\Phi},{d_\Phi}\in Z^+$ is a measurable function, we abbreviate $\mathcal H_{f_\Phi}$ with $\mathcal H_{\Phi}$. Then any finite subset $\mathcal G \subset \mathcal H_\Phi$ implies $\text{span}\{ 1, \mathcal G \} \subset \mathcal H_\Phi$, where $1$ denotes a constant function.
\end{proposition}
\begin{proof}
	For any $\gamma_1, \gamma_2\in\mathbb R$ and any $h_1, h_2 \in \mathcal H_f$, $\gamma_1 h_1 + \gamma_2 h_2 \in C(\mathcal X\times\mathcal Y)$.
	\begin{align}
		K_1(f,\gamma_1 h_1 + \gamma_2 h_2,P_S) &= 
		\int \Big|\mathbb E\left[(\gamma_1 h_1(X,Y) + \gamma_2 h_2(X,Y))(Y-v)\big|f(X)=v\right]\Big| dP_S(f^{-1}(v)) \\
		&\leq  \int \Big|\mathbb E\left[\gamma_1 h_1(X,Y)(Y-v)\big|f(X)=v\right]\Big| dP_S(f^{-1}(v)) \\
		&\;\;\;\;+ \int \Big|\mathbb E\left[\gamma_2 h_2(X,Y)(Y-v)\big|f(X)=v\right]\Big| dP_S(f^{-1}(v)) \\
		&= \gamma_1 K_1(f,h_1,P_S) + \gamma_2 K_1(f,h_2,P_S) \\
		&= 0.
	\end{align}
	According to Lemma~\ref{lem:app_k1k2}, $K_1(f,h,P_S)=0$ is equivalent as $K_2(f,h,P_S)=0$  for bounded $h$. Thus, $K_2(f,\gamma_1 h_1 + \gamma_2 h_2,P_S)=0$ which implies $\gamma_1 h_1 + \gamma_2 h_2 \in \mathcal H_f$. Now we finishes the proof that $\mathcal H_f$ is a linear space.
	
	For $f_{\Phi}(x)=\mathbb E[Y|\Phi(x)]$,
\begin{align}
	\mathbb E[Y|f_\Phi(X)] &= \mathbb E\left[ \mathbb E[Y|f_\Phi(X), \Phi(X)] \big|f_\Phi(X)\right] \\
	&= \mathbb E\left[ \mathbb E[Y|\Phi(X)] \big|f_\Phi(X)\right] \\
	&=\mathbb E[f_\Phi(X)|f_\Phi(X)] \\
	\label{eq:app_b10_1}
	&= f_\Phi(X).
\end{align} 
Thus, $K_2(f_\Phi,1,P_S)=0$ which implies $1\in \mathcal H_\Phi$. Since $\mathcal H_\Phi$ is a linear space, $\mathcal G \subset \mathcal H_\Phi$ implies $\text{span}\{ 1, \mathcal G \} \subset \mathcal H_\Phi$.
\end{proof}

\begin{lemma}
\label{lem:app_lem11}
	For any absolutely continuous probability measure $P(X,Y)$ with a continuous density function $p$, and any grouping function $h\in C(\mathcal X\times\mathcal Y)$, there exists $\gamma \neq 0$ and $\rho\in\mathbb R$ such that $\gamma h + \rho \in C(\mathcal X\times \mathcal Y)$, and it is density ratio between some absolutely continuous probability measure $P'$ and $P$, i.e., $dP'(x,y)=(\gamma h(x,y)+\rho)dP(x,y)$, where $P'$ also has a continuous density function.
\end{lemma}
\begin{proof}
	Since $h$ is bounded, there exists $\rho'\in\mathbb R$ such that $h(x,y)+\rho'>0$ for any $x\in\mathcal X, y\in\mathcal Y$. Define:
	\begin{align}
		\gamma &= \frac{1}{\int (h(x,y)+\rho') dP(x,y)}. \\
		\rho &= \frac{\rho'}{\int (h(x,y)+\rho') dP(x,y)}.
	\end{align}
	$\gamma h+\rho$ is still continuous. 
 
	
	We have $\gamma h(x,y)+\rho = \gamma (h(x,y)+\rho') > 0$ for any $x\in\mathcal X, y\in\mathcal Y$. 
	\begin{align}
		\int dP'(x,y) &= \int (\gamma h(x,y)+\rho)dP(x,y) \\
		&= \gamma\int ( h(x,y)+\rho')dP(x,y) \\
		&= 1.
	\end{align}
	Thus, $P'(X,Y)$ is an absolutely continuous probability measure.

        Its density function $p' = (\gamma h+\rho)p$ is continuous.
\end{proof}
\begin{theorem}
\label{theo:app_theo_Hphi}
	Consider an absolutely continuous probability measure $P_S(X,Y)$ and a predictor $f_{\Phi}(x)=\mathbb E_{P_S}[Y|\Phi(x)]$ where $\Phi:\mathbb R^{d} \rightarrow \mathbb R^{d_\Phi},{d_\Phi}\in Z^+$ is a measurable function.  We abbreviate $\mathcal H_{f_\Phi}$ with $\mathcal H_{\Phi}$. 
	\begin{align}
		\mathcal H_{\Phi} &= \left\{h\in C(\mathcal X\times\mathcal Y) : \text{Cov}_{P_S}\left[h(X,Y),Y|f_\Phi=v\right]=0 \;\;\text{for almost every}\; v\in [0,1]\right\} \\
		 &= \text{span}\left\{ \frac{p(x,y)} {p_S(x,y)}: \mathbb E_P[Y|f_\Phi]= \mathbb E_{P_S}[Y|f_\Phi]\;\;\text{almost surely}, \;\text{$p$ is continuous}\right\}.
	\end{align}
	
\end{theorem}
\begin{proof}
	First we prove:
	\begin{align}
		\mathcal H_{\Phi} = \left\{h\in C(\mathcal X\times \mathcal Y): \text{Cov}_{P_S}\left[h(X,Y)|f_\Phi=v\right]=0 \;\;\text{for almost every}\; v\in [0,1]\right\}.
	\end{align}	
	For each $v\in[0,1]$ and any $h\in C(\mathcal X\times \mathcal Y)$,
	\begin{align}
		\mathbb E_{P_S}\left[h(X,Y)(Y-v)\big|f_\Phi=v\right] 
	& = \mathbb E_{P_S}\left[ 
	h(X,Y)Y
	\Big|f_\Phi=v\right] - v\mathbb E_{P_S}\left[ 
	h(X,Y)
	\Big|f_\Phi=v\right] \\
	&= \mathbb E_{P_S}\left[ 
	h(X,Y)Y
	\Big|f_\Phi=v\right]  \\
	\label{eq:app_b12_1}
	&\;\;\;\;- \mathbb E_{P_S}\left[ 
	h(X,Y)
	\Big|f_\Phi=v\right]\mathbb E[Y|f_\Phi=v] \\
	&= \text{Cov}_{P_S}\left[h(X,Y),Y|f_\Phi=v\right].
	\end{align}
	Equation~\ref{eq:app_b12_1} follows from Equation~\ref{eq:app_b10_1}. 
	
	For any $h\in C(\mathcal X\times \mathcal Y)$,
	\begin{align}
		h\in \mathcal H_\Phi &\Leftrightarrow K_2(f_\Phi, h, P_S) = 0\\
		&\Leftrightarrow \int \Big(\mathbb E_{P_S}\left[h(X,Y)(Y-v)\big|f_\Phi=v\right]\Big)^2 dP_S(f_\Phi^{-1}(v)) \\
		&\Leftrightarrow \mathbb E_{P_S}\left[h(X,Y)(Y-v)\big|f_\Phi=v\right] = 0 \;\;\text{for almost every}\; v\in [0,1] \\
		&\Leftrightarrow \text{Cov}_{P_S}\left[h(X,Y),Y|f_\Phi=v\right]=0 \;\;\text{for almost every}\; v\in [0,1].
	\end{align}
%
	Next, we prove 
	\begin{align}
		\mathcal H_{\Phi}
		\supset \text{span}\left\{ \frac{p(x,y)} {p_S(x,y)}: \mathbb E_P[Y|f_\Phi]= \mathbb E_{P_S}[Y|f_\Phi]\;\;\text{almost surely},\;\text{$p$ is continuous}\right\}.
	\end{align}
	This is equivalent to saying that for any absolutely continuous probability measure $P$ satisfying $\mathbb E_P[Y|f_\Phi]= \mathbb E_{P_S}[Y|f_\Phi]$ almost surely, $\text{Cov}_{P_S}\left[p(X,Y)/p_S(X,Y),Y|f_\Phi=v\right]=0$ for almost every $v\in [0,1]$.
	\begin{align}
		&\;\;\;\;\text{Cov}_{P_S}\left[\frac{p(X,Y)}{p_S(X,Y)},Y\Big|f_\Phi=v\right]  \\
		&= \mathbb E_{P_S}\left[\frac{p(X,Y)}{p_S(X,Y)}Y \Big| f_\Phi=v \right] - \mathbb E_{P_S}\left[\frac{p(X,Y)}{p_S(X,Y)} \Big| f_\Phi=v \right] \mathbb E_{P_S}\left[Y \Big| f_\Phi=v \right] \\
		&= \int \frac{dP(f_{\Phi}^{-1}(v))}{dP_S(f_{\Phi}^{-1}(v))} \frac{p(x,y|f_\Phi=v)}{p_S(x,y|f_\Phi=v)} y \cdot dP_S(x,y|f_\Phi=v) \\
		& \;\;\;\; - \mathbb E_{P_S}\left[Y \Big| f_\Phi=v \right] \cdot \int \frac{dP(f_{\Phi}^{-1}(v))}{dP_S(f_{\Phi}^{-1}(v))} \frac{p(x,y|f_\Phi=v)}{p_S(x,y|f_\Phi=v)} \cdot dP_S(x,y|f_\Phi=v)   \\
		&=  \mathbb E_{P}\left[Y \Big| f_\Phi=v \right]\frac{dP(f_{\Phi}^{-1}(v))}{dP_S(f_{\Phi}^{-1}(v))} - \mathbb E_{P_S}\left[Y \Big| f_\Phi=v \right]\frac{dP(f_{\Phi}^{-1}(v))}{dP_S(f_{\Phi}^{-1}(v))} \\
		\label{eq:app_b12_3}
		&= \left[\mathbb E_{P}\left[Y \Big| f_\Phi=v \right] - \mathbb E_{P_S}\left[Y \Big| f_\Phi=v \right]\right]\frac{dP(f_{\Phi}^{-1}(v))}{dP_S(f_{\Phi}^{-1}(v))}    \\
		&= 0 \;\;\;\;\text{almost surely.}
	\end{align}
	Next, we prove
	\begin{align}
		\mathcal H_{\Phi}
		\subset \text{span}\left\{ \frac{p(x,y)} {p_S(x,y)}: \mathbb E_P[Y|f_\Phi]= \mathbb E_{P_S}[Y|f_\Phi]\;\;\text{almost surely},\;\text{$p$ is continuous}\right\}.
	\end{align}
	By Lemma~\ref{lem:app_lem11}, any grouping function $h\in \mathcal H_\Phi$ could be rewritten as $h(x,y) = \gamma p(x,y)/p_S(x,y)+\rho$ for some continuous density functions $p$ and $\gamma \neq 0$. Thus, we just need to prove the statement that $\text{Cov}_{P_S}\left[\gamma p(X,Y)/p_S(X,Y)+\rho,Y|f_\Phi=v\right]=0$  implies  $\mathbb E_P[Y|f_\Phi=v] = \mathbb E_{P_S}[Y|f_\Phi=v]$. 
	\begin{align}
		&\;\;\;\;\text{Cov}_{P_S}\left[\gamma \frac{p(X,Y)}{p_S(X,Y)}+\rho,Y\Big|f_\Phi=v\right] \\
		&= \gamma \mathbb E_{P_S}\left[\frac{p(X,Y)}{p_S(X,Y)}Y \Big| f_\Phi=v \right] + \rho E_{P_S}\left[Y \Big| f_\Phi=v \right] \\
		&\;\;\;\; - \gamma E_{P_S}\left[\frac{p(X,Y)}{p_S(X,Y)} \Big| f_\Phi=v \right]E_{P_S}\left[Y \Big| f_\Phi=v \right]  - \rho E_{P_S}\left[Y \Big| f_\Phi=v \right] \\
		&= \gamma \mathbb E_{P_S}\left[\frac{p(X,Y)}{p_S(X,Y)}Y \Big| f_\Phi=v \right] - \gamma E_{P_S}\left[\frac{p(X,Y)}{p_S(X,Y)} \Big| f_\Phi=v \right]E_{P_S}\left[Y \Big| f_\Phi=v \right] \\
		&= \gamma \text{Cov}_{P_S}\left[\frac{p(X,Y)}{p_S(X,Y)},Y\Big|f_\Phi=v\right] \\
		\label{eq:app_b12_2}
		&= \gamma \left[\mathbb E_{P}\left[Y \Big| f_\Phi=v \right] - \mathbb E_{P_S}\left[Y \Big| f_\Phi=v \right]\right]\frac{dP(f_{\Phi}^{-1}(v))}{dP_S(f_{\Phi}^{-1}(v))}.
	\end{align}
	Equation~\ref{eq:app_b12_2} follows from Equation~\ref{eq:app_b12_3}. So we have $\mathbb E_P[Y|f_\Phi=v] = \mathbb E_{P_S}[Y|f_\Phi=v]$ if $\text{Cov}_{P_S}\left[\gamma p(X,Y)/p_S(X,Y)+\rho,Y|f_\Phi=v\right]=0$.
\end{proof}
\begin{corollary}[Restatement of Theorem~\ref{theo:maximal_grouping_phi}]
	Consider an absolutely continuous probability measure $P_S(X,Y)$ and a calibrated predictor $f$.  
	\begin{align}
		\mathcal H_{f} &= \left\{h\in C(\mathcal X\times \mathcal Y): \text{Cov}_{P_S}\left[h(X,Y),Y|f=v\right]=0 \;\;\text{for almost every}\; v\in [0,1]\right\} \\
		 &= \text{span}\left\{ \frac{p(x,y)} {p_S(x,y)}: \mathbb E_P[Y|f]= \mathbb E_{P_S}[Y|f]\;\;\text{almost surely}, \;\text{$p$ is continuous}\right\}.
	\end{align}
	
\end{corollary}
\begin{remark}
    $\mathbb E_P[Y|f]= \mathbb E_{P_S}[Y|f]$ almost surely implies $\mathbb E_P[Y|f]=f$, which is equivalent as $R_P(f)=\inf_{g:[0,1]\rightarrow [0,1]} R_P(g\circ f)$, since we adopt square error.
\end{remark}
\begin{proof}
	Since $f$ is calibrated, we have $\mathbb E_{P_S}[Y|f]=f$. Take $\Phi(x)=f(x)$.
	\begin{align}
		f_\Phi(x) = \mathbb E_{P_S}[Y|\Phi(x)] = \mathbb E_{P_S}[Y|f(x)] = f(x).
	\end{align}
	Apply Theorem~\ref{theo:app_theo_Hphi} and the proof is complete.
\end{proof}
\begin{theorem}[Restatement of Theorem~\ref{theo:decomposeh1h2} (first part)]
\label{theo:app_h1h2_decompose}
	Consider an absolutely continuous probability measure $P_S(X,Y)$ and a predictor $f_{\Phi}(x)=\mathbb E_{P_S}[Y|\Phi(x)]$ where $\Phi:\mathbb R^{d} \rightarrow \mathbb R^{d_\Phi},{d_\Phi}\in Z^+$ is a measurable function. $\mathcal H_\Phi$ can be decomposed as $\mathcal H_\Phi = \mathcal H_{1,\Phi}+\mathcal H_{2,\Phi}$.
	\begin{align}
		\mathcal H_{1,\Phi} &:= \left\{h\in C(\mathcal X\times \mathcal Y): \text{Cov}_{P_S}\left[h(\Phi,Y),Y|f_\Phi=v\right]=0 \;\;\text{for almost every}\; v\in [0,1]\right\} \\
		&= \text{span}\left\{ \frac{p(\Phi,y)} {p_S(\Phi,y)}: \mathbb E_P[Y|f_\Phi]= \mathbb E_{P_S}[Y|f_\Phi]\;\;\text{almost surely},\;\text{$p$ is continuous}\right\}. \\
		\mathcal H_{2,\Phi} &:= \left\{h\in C(\mathcal X\times \mathcal Y):\mathbb E[h(X,Y)|\Phi,Y] \equiv C_h\;\; \forall \Phi,Y\right\} \\
		&= \text{span}\left\{ \frac{p(x|\Phi,y)} {p_S(x|\Phi,y)}: \;\text{$p$ is continuous}\right\}.
	\end{align}
\end{theorem}
\begin{remark}
    $\mathcal H_{1,\Phi}$ contains functions defined on $\text{supp}(\Phi)\times\mathcal Y$ which can be rewritten as functions on $\mathcal X\times \mathcal Y$ by variable substitution. Thus, $\mathcal H_{1,\Phi}$ is still a set of grouping functions. For $\mathcal H_{2,\Phi}$, $C_h$ is a constant depending on $h$. 
\end{remark}
\begin{proof}
	First we prove $\mathcal H_\Phi \subset \mathcal H_{1,\Phi}+\mathcal H_{2,\Phi}$. 
	
	For any $h_1 \in \mathcal H_{1,\Phi}$ and $h_2 \in \mathcal H_{2,\Phi}$, 
	\begin{align}
		&\;\;\;\;\text{Cov}_{P_S}\left[h_1(\Phi,Y)+h_2(X,Y),Y|f_\Phi=v\right] \\
		&=  \text{Cov}_{P_S}\left[h_1(\Phi,Y),Y|f_\Phi=v\right] + \text{Cov}_{P_S}\left[h_2(X,Y),Y|f_\Phi=v\right] \\
		&= \text{Cov}_{P_S}\left[h_2(X,Y),Y|f_\Phi=v\right] \\
		&= \mathbb E_{P_S}\left[ h_2(X,Y)Y | f_\Phi=v \right] - \mathbb E_{P_S}\left[ h_2(X,Y) | f_\Phi=v \right]\mathbb E_{P_S}\left[ Y | f_\Phi=v \right] \\
		&= \mathbb E_{P_S}\left[\mathbb E_{P_S}\left[h_2(X,Y)Y|\Phi,Y\right] \big| f_\Phi=v \right] - \mathbb E_{P_S}\left[\mathbb E_{P_S}\left[h_2(X,Y)|\Phi,Y\right] \big| f_\Phi=v \right]\mathbb E_{P_S}\left[ Y | f_\Phi=v \right] \\
		&= \mathbb E_{P_S}\left[\mathbb E_{P_S}\left[h_2(X,Y)|\Phi,Y\right]Y \big| f_\Phi=v \right] - \mathbb E_{P_S}\left[\mathbb E_{P_S}\left[h_2(X,Y)|\Phi,Y\right] \big| f_\Phi=v \right]\mathbb E_{P_S}\left[ Y | f_\Phi=v \right] \\
		&= C_{h_2} E_{P_S}\left[ Y | f_\Phi=v \right] - C_{h_2} E_{P_S}\left[ Y | f_\Phi=v \right] \\
		&= 0. 
	\end{align}
	Next we prove $\mathcal H_\Phi \supset \mathcal H_{1,\Phi}+\mathcal H_{2,\Phi}$.
	
	For any $h \in \mathcal H_\Phi$, let 
	\begin{align}
		h_1(\Phi(x),y) &= \mathbb E_{P_S}[h(X,Y)|\Phi=\Phi(x),Y=y]. \\
		h_2(x,y) &= h(x,y) - \mathbb E_{P_S}[h(X,Y)|\Phi=\Phi(x),Y=y].
	\end{align}
	Then we have
	\begin{align}
		&\;\;\;\;\text{Cov}_{P_S}\left[h_1(\Phi,Y),Y|f_\Phi=v\right] \\
		&= \mathbb E_{P_S}\left[ h_1(\Phi,Y)Y | f_\Phi=v \right] - \mathbb E_{P_S}\left[ h_1(\Phi,Y) | f_\Phi=v \right]\mathbb E_{P_S}\left[ Y| f_\Phi=v \right] \\
		&= \mathbb E_{P_S}\left[ \mathbb E_{P_S}[h(X,Y)|\Phi,Y]Y \big| f_\Phi=v \right] - 
		   \mathbb E_{P_S}\left[ \mathbb E_{P_S}[h(X,Y)|\Phi,Y] \big| f_\Phi=v \right]\mathbb E_{P_S}\left[ Y| f_\Phi=v \right]  \\
		&= \mathbb E_{P_S}\left[ h(X,Y)Y | f_\Phi=v \right] - \mathbb E_{P_S}\left[ h(X,Y) | f_\Phi=v \right]\mathbb E_{P_S}\left[ Y| f_\Phi=v \right] \\
		&= \text{Cov}_{P_S}\left[h(X,Y),Y|f_\Phi=v\right] \\
		&= 0 \;\;\;\;\text{for almost every}\; v\in [0,1].
	\end{align}
	Thus, $h_1(\Phi(x),y) \in \mathcal H_{1,\Phi}$.
	\begin{align}
		\mathbb E_{P_S}[h_2(X,Y)|\Phi,Y] &= \mathbb E_{P_S}\left[h(X,Y) - \mathbb E_{P_S}[h(X,Y)|\Phi,Y]\big|\Phi,Y\right] \\
		&= \mathbb E_{P_S}[h(X,Y)|\Phi,Y] - \mathbb E_{P_S}[h(X,Y)|\Phi,Y] \\
		&= 0.
	\end{align}
	Thus, $h_2(x,y) \in \mathcal H_{2,\Phi}$.
	Following a similar proof of Theorem~\ref{theo:app_theo_Hphi}, we have
	\begin{align}
		\mathcal H_{1,\Phi}= \text{span}\left\{ \frac{p(\Phi,y)} {p_S(\Phi,y)}: \mathbb E_P[Y|f_\Phi]= \mathbb E_{P_S}[Y|f_\Phi]\;\;\text{almost surely},\;\text{$p$ is continuous}\right\}.
	\end{align}
	Next, we prove $ \mathcal H_{2,\Phi} \supset \text{span}\left\{ p(x|\Phi,y) / p_S(x|\Phi,y): \;p\text{ is continuous} \right\}$. 
	
	This is equivalent to saying that $p(x|\Phi,y) / p_S(x|\Phi,y) \in \mathcal H_{2,\Phi}$ for any continuous density function $p$.
	 \begin{align}
	 	\mathbb E_{P_S}\left[ \frac{p(X|\Phi,Y)}{p_S(X|\Phi,Y)} \Big| \Phi,Y  \right] &= 
	 	\int \frac{p(x|\Phi,Y)}{p_S(x|\Phi,Y)} dP_S(x|\Phi,Y)  \\
	 	&= \int dP(x|\Phi,Y) \\
	 	&= 1.
	 \end{align}
	 Thus, we have $p(x|\Phi,y) / p_S(x|\Phi,y) \in \mathcal H_{2,\Phi}$.
	 
	 Next, we prove $ \mathcal H_{2,\Phi} \subset \text{span}\left\{ p(x|\Phi,y) / p_S(x|\Phi,y): \;p\text{ is continuous} \right\}$.

	 By Lemma~\ref{lem:app_lem11}, any grouping function $h\in \mathcal H_{2,\Phi}$ could be rewritten as $h_2(x,y) = \gamma p(x,y)/p_S(x,y)+\rho$ for some continuous density function $p$ and $\gamma \neq 0$. Thus, we just need to prove the statement that $\mathbb E_{P_S}[\gamma p(X,Y)/p_S(X,Y)+\rho|\Phi,Y]\equiv C_{h_2}$  implies  $p(X,Y)/p_S(X,Y)\equiv p(X|\Phi,Y)/p_S(X|\Phi,Y)$. 
	 \begin{align}
	 	\mathbb E_{P_S}\left [\gamma \frac{p(X,Y)}{p_S(X,Y)} + \rho \Big|\Phi,Y \right] &= \mathbb E_{P_S}\left [\gamma \frac{p(X|\Phi,Y)}{p_S(X|\Phi,Y)} \frac{p(\Phi,Y)}{p_S(\Phi,Y)} + \rho \Big|\Phi,Y \right]  \\
	 	&= \gamma\frac{p(\Phi,Y)}{p_S(\Phi,Y)} \mathbb E_{P_S}\left [ \frac{p(X|\Phi,Y)}{p_S(X|\Phi,Y)}  \Big|\Phi,Y \right] + \rho \\
	 	&= \gamma\frac{p(\Phi,Y)}{p_S(\Phi,Y)} \mathbb E_P\left [ 1  \Big|\Phi,Y \right] + \rho \\
	 	&=  \gamma\frac{p(\Phi,Y)}{p_S(\Phi,Y)} + \rho.
	 \end{align}
	 Thus, we have $p(\Phi,Y)/p_S(\Phi,Y) \equiv (C_{h_2}-\rho)/\gamma$ which is a constant. Since $\mathbb E_{P_S}[p(\Phi,Y)/p_S(\Phi,Y)] = 1$, we have $p(\Phi,Y)/p_S(\Phi,Y) \equiv 1$.
	 \begin{align}
	 	\frac{p(X,Y)}{p_S(X,Y)} &= \frac{p(X|\Phi,Y)}{p_S(X|\Phi,Y)} \frac{p(\Phi,Y)}{p_S(\Phi,Y)} \\
	 	&=  \frac{p(X|\Phi,Y)}{p_S(X|\Phi,Y)}.
	 \end{align}
	 
\end{proof}
\begin{lemma}[Theorem 3.2 from \citet{MCBoost}]
\label{lem:app_risk_imply_mc}
    If $f$ is calibrated and there exists an $h(x)$ such that
    \begin{align}
    \mathbb E[(f(X)-Y)^2 - (h(X)-Y)^2|f(X)=v] \geq \alpha,
    \end{align}
    then:
    \begin{align}
    \mathbb E[h(X)(Y-v)|f(X)=v] \geq \frac{\alpha}{2}.
    \end{align}
\end{lemma}
\begin{proposition}[Restatement of Theorem~\ref{theo:decomposeh1h2} (second part)]
    If a predictor $f$ is multicalibrated with $\mathcal H_{1,\Phi}$, then $R_{P_S}(f) \leq R_{P_S}(f_\Phi)$.
\end{proposition}
\begin{proof}
    We prove by contradiction. If $R_{P_S}(f) > R_{P_S}(f_\Phi)$, then
    \begin{align}
        \int \mathbb E\left[(f(X)-Y)^2 - (f_\Phi(X)-Y)^2|f(X)=v\right] dP_S(f^{-1}(v)) > 0.
    \end{align}
    Let $\alpha_v = \mathbb E\left[(f(X)-Y)^2 - (h(X)-Y)^2|f(X)=v\right]$.

    Since $f$ is multicalibrated with $\mathcal H_{1,\Phi}$, $f$ is calibrated. It follows from Lemma~\ref{lem:app_risk_imply_mc}:
    \begin{align}
    \mathbb E[f_\Phi(X)(Y-v)|f(X)=v] \geq \frac{\alpha_v}{2}.
    \end{align}
    Then,
    \begin{align}
        K_1(f, f_\Phi, P_S) &= \int \Big|E[f_\Phi(X)(Y-v)|f(X)=v] \Big| dP_S(f^{-1}(v)) \\
        &\geq \int \alpha_v dP_S(f^{-1}(v)) \\
        &>0.
    \end{align}
    From Lemma~\ref{lem:app_k1k2}, we have $K_2(f, f_\Phi, P_S) \geq K_1(f, f_\Phi, P_S)^2 > 0$.
    Since $f_\Phi \in \mathcal H_{1,\Phi}$, it contradicts with the fact that $f$ is multicalibrated with $\mathcal H_{1,\Phi}$.
\end{proof}
\begin{proposition}[Restatement of Theorem~\ref{theo:decomposeh1h2} (third part)]
    $f_\Phi$ is an invariant predictor elicited by $\Phi$ across a set of environments $\mathscr E$ where $P_e(\Phi,Y)=P_S(\Phi,Y)$ for any $e\in\mathscr E$.  If a predictor $f$ is multicalibrated with $\mathcal H_{2,\Phi}$,  then $f$ is also an invariant predictor across $\mathscr E$ elicited by some representation.
\end{proposition}
\begin{proof}
    Since $P_e(\Phi,Y)=P_S(\Phi,Y)$, we have $\mathbb E_{P_e}[Y|\Phi] = \mathbb E_{P_S}[Y|\Phi] = f_\Phi$ for every $e\in\mathscr E$. 
    
    Thus, 
    \begin{align}
        R_{P_e}(f_\Phi) = \inf_{g \in \mathcal G} { R_{P_e}(g\circ \Phi)}.
    \end{align}
    This implies $f_\Phi$ is an invariant predictor elicited by $\Phi$ across $\mathscr E$.

    For any $e\in\mathscr E$, we have:
    \begin{align}
        \frac{p_e(x,y)}{p_S(x,y)} = \frac{p_e(\Phi,y)}{p_S(\Phi,y)} \frac{p_e(x|Phi,y)}{p_S(x|\Phi,y)} = \frac{p_e(x|\Phi,y)}{p_S(x|\Phi,y)} \in \mathcal H_{2,\Phi}.
    \end{align}
    Since $f$ is multicalibrated with $\mathcal H_{2,\Phi}$, it follows from Theorem~\ref{theo:app_approx_mc_inv}:
    \begin{align}
        R_{P_e}(f) = \inf_{g:[0,1]\rightarrow[0,1]}  R_{P_e}(g\circ f).
    \end{align}
    This implies that $f$ is an invariant predictor across $\mathscr E$ elicited by $f$.
\end{proof}
\begin{proposition}[Restatement of Proposition~\ref{prop:monoticity}]
\label{prop:app_monoticity}
	Consider $X\in\mathbb R^d$ which could be sliced as $X=(\Phi,\Psi)^T$ and $\Phi=(\Lambda,\Omega)^T$. Define $\mathcal H_{1,\Phi}' := \{ h(\Phi(x)) \in C(\mathcal X\times \mathcal Y) \}$. $\mathcal H_{1,X}'$ and $\mathcal H_{1,\Lambda}'$ are similarly defined.
	We have:
	
1. $\mathcal H_{1,X}' \supset \mathcal H_{1,\Phi}' \supset \mathcal H_{1,\Lambda}' \supset \mathcal H_{1,\emptyset}' = \{ C\}.$

2. $\{C\}=\mathcal H_{2,X} \subset H_{2,\Phi} \subset \mathcal H_{2,\Lambda} \subset \mathcal H_{2,\emptyset}.$

$C$ is a constant value function.
\end{proposition}
\begin{remark}
    The proposition shows that $\mathcal H_1'$, as a subspace of $\mathcal H_1$, evolves monotonically and in opposite direction to $\mathcal H_2$. If we perceive the representation $\Phi$ as a filter, gaining more information from covariates facilitates multicalibration w.r.t. $\mathcal H_1$ (and accuracy) but hampers multicalibration w.r.t. $\mathcal H_2$ (and invariance). With $\mathcal H_1'$ and $\mathcal H_2$ combined together, a multicalibrated predictor is searching for an appropriate level of information filter to balance the tradeoff between accuracy and invariance.
\end{remark}
\begin{proof}
	\;\\
1. For any $h(\Lambda) \in \mathcal H_{1,\Lambda}'$, since $\Phi=(\Lambda,\Omega)$, $h(\Lambda)$ is also a function of $\Phi$. Thus, we have $h(\Lambda) \in \mathcal H_{1,\Phi}'$. It follows that $\mathcal H_{1,\Phi}' \supset \mathcal H_{1,\Lambda}'$. Similarly we have $\mathcal H_{1,X}' \supset \mathcal H_{1,\Phi}'$ and $\mathcal H_{1,\Lambda}' \supset \mathcal H_{1,\emptyset}'$.

2. For any $h(\Lambda,\Omega,\Psi,Y) \in \mathcal H_{2,\Phi}$ such that $\mathbb E[h(\Phi,\Psi,Y)|\Phi,Y] = C_h$ for any values of $\Phi,Y$,
we have
\begin{align}
	\mathbb E[h(\Lambda,\Omega,\Psi,Y)|\Lambda,Y] &= \mathbb E\left[ \mathbb E[h(\Lambda,\Omega,\Psi,Y)|\Lambda,\Omega,Y] \big| \Lambda,Y\right] \\
	&= \mathbb E\left[ \mathbb E[h(\Phi,\Psi,Y)|\Phi,Y] \big| \Lambda,Y\right] \\
	&= \mathbb E[C_h|\Lambda,Y] \\
	&= C_h \;\;\;\;\text{for any values of }\Lambda, Y.
\end{align}
Thus, $h(\Lambda,\Omega,\Psi,Y) \in \mathcal H_{2,\Lambda}$. It follows that $H_{2,\Phi} \subset \mathcal H_{2,\Lambda}$. Similarly, we have $\mathcal H_{2,X} \subset H_{2,\Phi}$ and $\mathcal H_{2,\Lambda} \subset \mathcal H_{2,\emptyset}$. Particularly for $h(x,y) \in \mathcal H_{2,X}$, we have $h(X,Y) = \mathbb E[h(X,Y)|X,Y]$ is a constant for any values of $X,Y$.
\end{proof}
\begin{lemma}[\citet{MCBoost}]
\label{app:lem_when_mc_accuracy}
	Let $\mathcal H \subset \mathcal X^\mathbb R$ be such a grouping function class that  $h\in \mathcal H$ implies $\gamma h + \rho \in \mathcal H$ for any $h\in\mathcal H$ and $\gamma,\rho \in\mathbb R$.
	If $\mathcal H$ satisfies the $(0,0)$-weak learning condition in Assumption~\ref{app_assum:app_sufficient_grouping}, a predictor $f$ is multicalibrated w.r.t. $\mathcal H$ if and only if $f(x) = \mathbb E[Y|x]$ almost surely.
\end{lemma}
\begin{proposition}
	For a measurable function $\Phi:\mathbb R^{d} \rightarrow \mathbb R^{d_\Phi},{d_\Phi}\in Z^+$, a predictor $f:\text{supp}(\Phi) \rightarrow [0,1]$ is multicalibrated w.r.t. $\mathcal H_{1,\Phi}$ if and only if $f$ is multicalibrated w.r.t. $\mathcal H_{1,\Phi}'$.
\end{proposition}
\begin{proof}
	Since $\mathcal H_{1,\Phi}' \subset \mathcal H_{1,\Phi}$, $f$'s multicalibration w.r.t. $\mathcal H_{1,\Phi}$ implies $f$'s multicalibration w.r.t. $\mathcal H_{1,\Phi}'$.
	
	On the other hand, $\mathcal H_{1,\Phi}'$ satisfies the $(0,0)$-weak learning condition with the pushforward measure on $\Phi$, because $\mathbb E[Y|\Phi] \in \mathcal H_{1,\Phi}'$. It follows from Lemma~\ref{app:lem_when_mc_accuracy} that $f$ is multicalibrated w.r.t. $\mathcal H_{1,\Phi}'$ implies $f(\Phi)=\mathbb E[Y|\Phi]$ almost surely. By the definition of $\mathcal H_{1,\Phi}$, $f(\Phi)$ is multicalibrated w.r.t. $\mathcal H_{1,\Phi}$.
	\end{proof}

\subsection{MC-PseudoLabel: An Algorithm for Extended Multicalibration}

\begin{lemma}
\label{lemm:app_stop_cri}
Fix a model \( f : \mathcal{X} \rightarrow [0,1] \). Suppose for some \( v \in \text{Range}(f) \) there is an \( h \in \mathcal{H} \) such that:
\[
\mathbb{E}[h(x,y)(y - v) | f(x) = v] > \alpha
\]

Let \( h' = v + \eta h(x,y) \) for \( \eta = \frac{\alpha}{\mathbb{E}[h(x,y)^2 | f(x) = v]} \). 

Then:
\[
\mathbb{E}[(f(x) - y)^2 - (h'(x,y) - y)^2 | f(x) = v] > \frac{\alpha^2}{\mathbb{E}[h(x,y)^2 | f(x) = v]}.
\]

\textit{Proof.} Following~\citep{MCBoost}, we have
\begin{align*}
\mathbb{E}[(f(x) - y)^2 &- (h'(x,y) - y)^2 | f(x) = v] \\
&= \mathbb{E}[(v - y)^2 - (v + \eta h(x,y) - y)^2 | f(x) = v] \\
&= \mathbb{E}[v^2 - 2vy + y^2 - (v + \eta h(x,y))^2 + 2y(v + \eta h(x,y)) - y^2 | f(x) = v] \\
&= \mathbb{E}[2\eta yh(x,y) - 2\eta vh(x,y) - \eta^2 h(x,y)^2 | f(x) = v] \\
&= \mathbb{E}[2\eta h(x,y)(y - v) - \eta^2 h(x,y)^2 | f(x) = v] \\
&> 2\eta \alpha - \eta^2 \mathbb{E}[h(x,y)^2 | f(x) = v] \\
&= \frac{\alpha^2}{\mathbb{E}[h(x,y)^2 | f(x) = v]}.
\end{align*}
\end{lemma}

\begin{theorem}[Restatement of Theorem~\ref{theo:stopping_criteria}]
\label{theo:app_stopping_criteria}
	In Algorithm \ref{alg}, for $\alpha,B>0$, if the following is satisfied:
\begin{align}
	\label{eq:app_stopping_criteria}
	err_{t-1} - \tilde{err}_t \leq \frac{\alpha}{B},
\end{align}
	the output $f_{t-1}'(x)$ is $\alpha$-approximately $\ell_2$ multicalibrated w.r.t. $\mathcal H_B=\{h\in \mathcal H:\sup h(x,y)^2\leq B\}$.
\end{theorem}
\begin{proof}
We prove by contradiction. Assume that \( f_{t-1} \) is not \( \alpha \)-approximately multicalibrated with respect to \( \mathcal{H}_B \). Then there exists \( h \in \mathcal{H}_B \) such that:

\[
\sum_{v \in [1/m]} P(f_{t-1}(x) = v) \left( \mathbb{E} \left[ h(x,y)(y - v)\Big|f_{t-1}(x) = v \right] \right)^2 > \alpha.
\]

For each \( v \in [1/m] \) define

\[
\alpha_v := P(f_{t-1}(x) = v) \left( \mathbb{E} \left[ h(x,y)(y - v)\Big|f_{t-1}(x) = v \right] \right)^2.
\]

Then we have \( \sum_{v\in[1/m]} \alpha_v > \alpha \).

According to Lemma~\ref{lemm:app_stop_cri}, for each \( v \in [1/m] \), there exists \( h_v \in \mathcal{H} \) such that:
\begin{align}
	&\mathbb{E}\left[(f_{t-1}(x) - y)^2 - (h_v(x,y) - y)^2 | f_{t-1}(x) = v\right] \\
>& \;\frac{\alpha_v}{\mathbb{E}[h(x)^2|f_{t-1}(x) = v] \cdot P(f_{t-1}(x) = v)} \\
\geq&\; \frac{\alpha_v}{B\cdot P(f_{t-1}(x) = v)}.
\end{align}
Then,

\begin{align*}
&\;\;\;\;\mathbb{E} \left[ (f_{t-1}(x) - y)^2 - (\tilde{f}_t(x) - y)^2 \right] \\
&= \sum_{v \in [1/m]} P(f_{t-1}(x) = v) \mathbb{E} \left[ (f_{t-1}(x) - y)^2 - (\tilde{f}_t(x) - y)^2 | f_{t-1}(x) = v \right] \\
&= \sum_{v \in [1/m]} P(f_{t-1}(x) = v) \mathbb{E} \left[ (f_{t-1}(x) - y)^2 - (h_v^t(x,y) - y)^2 | f_{t-1}(x) = v \right] \\
&> \frac{\alpha}{B},
\end{align*}
which contradicts the condition in Equation~\ref{eq:app_stopping_criteria}.
\end{proof}

The following proposition is a direct corollary from \citet{MCBoost}'s Theorem 4.3.
\begin{proposition}
\label{prop:app_converge_h1}
    For any distribution $D$ supported on $\mathcal X \times \mathcal Y$ and $\Phi \in \sigma(X)$ ,  take the grouping function class $\mathcal H \subset \mathcal H'_{1,\Phi}:=\{ h(\Phi(x)) \in C(\mathcal X\times \mathcal Y) \}$ and the predictor class $\mathcal F=\mathbb R^{\mathcal X}$. For any $0<\alpha<1, B>0$ and an initial predictor $f_0:\mathcal X\rightarrow [0,1]$ with $|\text{Range}(f_0)|\geq \frac{2B}{\alpha}$, then $\text{MC-Pseudolabel}(D, \mathcal H, \mathcal F)$ halts after at most $T \leq \frac{2B}{\alpha}$ steps and outputs a model $f_{T-1}'(x)$ that is $\alpha$-approximately $\ell_2$ multicalibrated w.r.t $D$ and $\mathcal H_B=\{h\in \mathcal H:\sup h(x,y)^2\leq B\}$.
\end{proposition}
\begin{theorem}[Restatement of Theorem~\ref{theo:converge_h2}]
\label{theo:app_converge_h2}
	Consider $X\in\mathbb R^d$ with $X=(\Phi,\Psi)^T$. Assume that $(\Phi,\Psi,Y)$ follows a multivariate normal distribution $\mathcal N_{d+1}(\mu,\Sigma)$ where the random variables are in general position such that $\Sigma$ is positive definite. we partition $\Sigma$ into blocks:
 \begin{align}
     \Sigma = 
\begin{pmatrix}
\Sigma_{\Phi\Phi} & \Sigma_{\Phi\Psi} & \Sigma_{\Phi y} \\
\Sigma_{\Psi\Phi} & \Sigma_{\Psi\Psi} & \Sigma_{\Psi y} \\
\Sigma_{y\Phi} & \Sigma_{y\Psi} & \Sigma_{yy} 
\end{pmatrix}.
 \end{align}
 For any distribution $D$ supported on $\mathcal X \times \mathcal Y$, take the grouping function class $\mathcal H=\{h\in\mathcal H_{2,\Phi}:h(x,y)=c_x^Tx+c_yy + c_b, c_x\in\mathbb R^d,c_y, c_b\in\mathbb R \}$ and the predictor class $\mathcal F=\mathbb R^{\mathcal X}$. For an initial predictor $f^{(0)}(x)=\mathbb E[Y|x]$, run $\text{MC-Pseudolabel}(D, \mathcal H, \mathcal F)$ without rounding, 
then $f^{(t)}(x) \rightarrow \mathbb E[Y|\Phi(x)]$ as $t \rightarrow \infty$ with a convergence rate of $\mathcal O(M(\Sigma)^t)$, where
\begin{align}
    M(\Sigma) &= 
( \Sigma_{yy}- \Sigma_{y\Phi}\Sigma_{\Phi\Phi}^{-1}\Sigma_{\Phi y})^{-1}
( \Sigma_{y\Psi}- \Sigma_{y\Phi}\Sigma_{\Phi\Phi}^{-1}\Sigma_{\Phi \Psi}) \\
&\;\;\;\; ( \Sigma_{\Psi\Psi}- \Sigma_{\Psi\Phi}\Sigma_{\Phi\Phi}^{-1}\Sigma_{\Phi \Psi})^{-1}
( \Sigma_{\Psi y}- \Sigma_{\Psi\Phi}\Sigma_{\Phi\Phi}^{-1}\Sigma_{\Phi y}).
\end{align}
We have $0\leq M(\Sigma) <1$.
\end{theorem}
\begin{remark}
    $f^{(t)}$ are always linear. Thus, the limit of functions is equivalent as the limit of coefficients.
\end{remark}
\begin{remark}
    $\mathbb E[Y|\Phi(x)]$ is multicalibrated with respect to $\mathcal H$. Furthermore, any calibrated predictor on $\Phi$, denoted by $g(\Phi)$, is multicalibrated with respect to $\mathcal H$. This is because:
    \begin{align}
        \mathbb E[h(X,Y)(Y-g(\Phi)) | g(\Phi)] &= \mathbb E\left[ \mathbb E\left[ h(X,Y)(Y-g(\Phi)) | \Phi, Y \right] | g(\Phi)\right] \\
        &= \mathbb E\left[ C_h (Y-g(\Phi)) | g(\Phi)\right] \\
        &= 0.
    \end{align}
    However, $\mathbb E[Y|\Phi]$ is the most accurate predictor among all multicalibrated $g(\Phi)$.
\end{remark}
\begin{remark}
    The convergence rate $M(\Sigma)$ does not depend on the dimension $d$ of covariates. When $Y\perp \Psi\;|\;\Phi$ implying that $\Phi$ is sufficient for prediction, following from $\mathbb E[Y|\Phi,\Psi] = \mathbb E[Y|\Phi]$: 
    \begin{align}
        \Sigma_{\Psi y} = \Sigma_{\Psi \Phi}\Sigma_{\Phi \Phi}^{-1}\Sigma_{\Phi y}.
    \end{align}
    It follows that $M(\Sigma)=0$ and the algorithm will converge in one step. 
    
    On the other hand, when $Y$ and $\Psi$ are linearly dependent given $\Phi$ such that $\Sigma$ is singular, which violates positive definiteness, following from the proof below:
    \begin{align}
        \Sigma_{yy}- \Sigma_{y\Phi}\Sigma_{\Phi\Phi}^{-1}\Sigma_{\Phi y} = 
        (\Sigma_{y\Psi}- \Sigma_{y\Phi}\Sigma_{\Phi\Phi}^{-1}\Sigma_{\Phi \Psi})
( \Sigma_{\Psi \Psi}- \Sigma_{\Psi\Phi}\Sigma_{\Phi \Phi}^{-1}\Sigma_{\Phi \Psi})^{-1}
( \Sigma_{\Psi y}- \Sigma_{\Psi\Phi}\Sigma_{\Phi \Phi}^{-1}\Sigma_{\Phi y}).
    \end{align}
    It follows that $M(\Sigma)=1$ and the algorithm can't converge. 
    
    So the convergence rate depends on the singularity of the problem. Since the algorithm converges to a predictor that does not depend on $\Psi$, stronger the "spurious" correlation between $Y$ and $\Psi$ given $\Phi$ in the distribution $D$, the algorithm takes longer to converge.
\end{remark}
\begin{proof}
    Without loss of generality, assume $\mu=0$.
    
    Let $c=(c_x, c_y, c_b)^T = (c_\Phi, c_\Psi, c_y, c_b)^T$.   Denote dimensions of $\Phi,\Psi$ by $d_\Phi,d_\Psi$.

    Let $\mathbb E[Y|\Phi,\Psi] = (\alpha_\Phi^\star)^T\Phi + (\alpha_\Psi^\star)^T\Psi$ and $\mathbb E[\Psi|\Phi,Y] = (\beta_\Phi^\star)^T\Phi + (\beta_y^\star)^TY$ with 
    $\alpha_\Phi^\star \in \mathbb R^{d_\Phi}, $ \\
    $\alpha_\Psi^\star \in \mathbb R^{d_\Psi}, \beta_\Phi^\star \in \mathbb R^{d_\Phi \times d_\Psi}, \beta_y^\star \in \mathbb R^{1 \times d_\Psi}$.
    
     We have:
     \begin{align}
        \begin{pmatrix}
        \alpha_\Phi^\star \\
        \alpha_\Psi^\star
        \end{pmatrix} &=
        \begin{pmatrix}
        \Sigma_{\Phi\Phi} & \Sigma_{\Phi\Psi} \\
        \Sigma_{\Psi\Phi} & \Sigma_{\Psi\Psi} 
        \end{pmatrix}^{-1}
        \begin{pmatrix}
        \Sigma_{\Phi y} \\
        \Sigma_{\Psi y}
        \end{pmatrix}
        . \\
        \begin{pmatrix}
        \beta_\Phi^\star \\
        \beta_y^\star
        \end{pmatrix} &=
        \begin{pmatrix}
        \Sigma_{\Phi\Phi} & \Sigma_{\Phi y} \\
        \Sigma_{y\Phi} & \Sigma_{yy} 
        \end{pmatrix}^{-1}
        \begin{pmatrix}
        \Sigma_{\Phi \Psi} \\
        \Sigma_{y\Psi}
        \end{pmatrix}
        .
     \end{align}
    According to Theorem~\ref{theo:app_h1h2_decompose}, $\mathbb E[h(X,Y)|\Phi,Y]$ is a constant for different values of $\Phi,Y$.
    \begin{align}
        \mathbb E[h(X,Y)|\Phi,Y] 
        &= c_\Phi^T\Phi + c_y^TY 
        + c_\Psi^T \mathbb E[\Psi|\Phi,Y] + c_b \\
        &= c_\Phi^T\Phi + c_y^TY 
        + c_\Psi^T \begin{pmatrix}
                \Sigma_{\Psi\Phi} & \Sigma_{\Psi y}
                \end{pmatrix}
                \begin{pmatrix}
                \Sigma_{\Phi\Phi} & \Sigma_{\Phi y} \\
                \Sigma_{y\Phi} & \Sigma_{yy} \\
                \end{pmatrix}^{-1}
                \begin{pmatrix}
                \Phi \\
                Y \\
                \end{pmatrix} 
                + c_b.
    \end{align}
    This implies:
    \begin{align}
        \label{eq:app_convergence_h2_12}
        \begin{pmatrix}
        \Sigma_{\Phi\Phi} & \Sigma_{\Phi y} \\
        \Sigma_{y\Phi} & \Sigma_{yy} \\
        \end{pmatrix}^{-1}
        \begin{pmatrix}
        \Sigma_{\Phi\Psi} \\
        \Sigma_{y\Psi}
        \end{pmatrix}
        c_\Psi
        +
        \begin{pmatrix}
        c_\Phi \\
        c_y
        \end{pmatrix}
        =0.
    \end{align}
    Rearranging to:
    \begin{align}
        \label{eq:app_convergence_h2_11}
        \begin{pmatrix}
        \Sigma_{\Phi\Phi} & \Sigma_{\Phi\Psi} & \Sigma_{\Phi y} \\
        \Sigma_{y\Phi} & \Sigma_{y\Psi} & \Sigma_{yy} 
        \end{pmatrix}
        c = 0.
    \end{align}
    Let $f^{(t)} = (\alpha_\Phi^{(t)})^T\Phi + (\alpha_\Psi^{(t)})^T\Psi$ and $\tilde f^{(t)} = (\tilde\alpha_\Phi^{(t)})^T\Phi + (\tilde\alpha_\Psi^{(t)})^T\Psi + (\tilde\alpha_y^{(t)})^TY$.  We claim that $\tilde\alpha_\Psi^{(t)} = 0$. Otherwise, consider $\tilde f^{(t)'} = (\tilde\alpha_\Phi^{(t)})^T\Phi + (\tilde \alpha_\Psi^{(t)})^T\mathbb E[\Psi|\Phi,Y] + (\tilde\alpha_y^{(t)})^TY$.
    \begin{align}
        \mathbb E[\tilde f^{(t)} - \tilde f^{(t)'}|\Psi, Y] = \mathbb E[(\tilde\alpha_\Psi^{(t)})^T \Psi|\Phi,Y] - (\tilde\alpha_\Psi^{(t)})^T\mathbb E[\Psi|\Phi,Y]
        = 0.
    \end{align}
    Thus, $\tilde f^{(t)} - \tilde f^{(t)'} \in \mathcal H$. 
    
    On the other hand,
    \begin{align}
        &\;\;\;\; \mathbb E[(Y-\tilde f^{(t)})^2] - \mathbb E[(Y-\tilde f^{(t)'})^2] \\
        &= \mathbb E\left[ (\tilde\alpha_\Psi^{(t)})^T \left[ \mathbb E[\Psi\Psi^T|\Phi,Y]  -  \mathbb E[\Psi|\Phi,Y] \mathbb E[\Psi^T|\Phi,Y] \right]\tilde\alpha_\Psi^{(t)}  \right]  \\
        &> 0.
    \end{align}
    The inequality follows from the fact that $\mathbb E[\Psi\Psi^T|\Phi,Y]  -  \mathbb E[\Psi|\Phi,Y] \mathbb E[\Psi^T|\Phi,Y]$ is the covariance matrix of $\Psi|\Phi, Y$, which is positive definite because $\Sigma$ is positive definite.
    The inequality contradicts with the definition of $\tilde f^{(t)}$. Thus, $\tilde f^{(t)} = (\tilde\alpha_\Phi^{(t)})^T\Phi + (\tilde\alpha_y^{(t)})^TY$.

    Define a matrix $C\in\mathbb R^{(d+1)\times d_t}$ whose columns support the solution space of Equation~\ref{eq:app_convergence_h2_11}. Then $H:=C^T(S,T,Y)^T \in\mathbb R^{d_t}$ is a random vector.
    According to the definition of $\tilde f^{(t)}$,
    \begin{align}
        \label{eq:app_convergece_1}
        \tilde f^{(t+1)} &= \mathbb E[Y|f^{(t)}, H] \\
        \label{eq:app_convergence_13}
        &= k^{(t)}f^{(t)}  + (c_\Phi^{(t)})^T\Phi +  (c_\Psi^{(t)})^T \Psi + (c_y^{(t)})^T Y \\
        \label{eq:app_convergence_15}
        &= k^{(t)} (\alpha_\Phi^{(t)})^T\Phi + k^{(t)} (\alpha_\Psi^{(t)})^T\Psi + (c_\Phi^{(t)})^T\Phi +  (c_\Psi^{(t)})^T \Psi + (c_y^{(t)})^T Y.
    \end{align}
    In the above equation, $k^{(t)} \in \mathbb R$.
    
    Since $\tilde\alpha_\Psi^{(t+1)} = k^{(t)} \alpha_\Psi^{(t)} +c_\Psi^{(t)}  = 0$, we have $c_\Psi^{(t)} = -k^{(t)} \alpha_\Psi^{(t)} $. Substituting into Equation~\ref{eq:app_convergence_h2_12}:
    \begin{align}
        \begin{pmatrix}
        c_\Phi^{(t)} \\
        c_y^{(t)}
        \end{pmatrix}
        = k^{(t)}
        \begin{pmatrix}
        \Sigma_{\Phi\Phi} & \Sigma_{\Phi y} \\
        \Sigma_{y\Phi} & \Sigma_{yy} \\
        \end{pmatrix}^{-1}
        \begin{pmatrix}
        \Sigma_{\Phi\Psi} \\
        \Sigma_{y\Psi}
        \end{pmatrix}
        \alpha_\Psi^{(t)}.
    \end{align}
    Substituting into Equation~\ref{eq:app_convergence_15}: 
     \begin{align}
     \begin{pmatrix}
     \tilde\alpha_\Phi^{(t+1)} \\
     \tilde\alpha_y^{(t+1)}
     \end{pmatrix} &= 
     k^{(t)} 
     \begin{pmatrix}
     \tilde\alpha_\Phi^{(t)} \\
     0
     \end{pmatrix} 
     + 
     \begin{pmatrix}
        c_\Phi^{(t)} \\
        c_y^{(t)}
        \end{pmatrix}
     \\
     &= k^{(t)}
     \begin{pmatrix}
     I_{d_\Phi} & \beta_\Phi^\star\\
     0 & \beta_y^\star
     \end{pmatrix}
     \begin{pmatrix}
     \alpha_\Phi^{(t)} \\
     \alpha_\Psi^{(t)}
     \end{pmatrix}.
     \end{align}
    Then we have $f^{(t+1)} = \mathbb E[\tilde f^{(t+1)} | \Phi,\Psi] = (\tilde \alpha_\Phi^{(t+1)})^T\Phi +   (\tilde \alpha_y^{(t+1)})^T \mathbb E[Y|\Phi,\Psi]$. 
    
    This is equivalent as:
    \begin{align}
     \begin{pmatrix}
     \alpha_\Phi^{(t+1)} \\
     \alpha_\Psi^{(t+1)}
     \end{pmatrix} &= 
     \begin{pmatrix}
     \tilde\alpha_\Phi^{(t+1)} \\
     0
     \end{pmatrix} +
        \begin{pmatrix}
        \Sigma_{\Phi\Phi} & \Sigma_{\Phi\Psi} \\
        \Sigma_{\Psi\Phi} & \Sigma_{\Psi\Psi} 
        \end{pmatrix}^{-1}
        \begin{pmatrix}
        \Sigma_{\Phi y} \\
        \Sigma_{\Psi y}
        \end{pmatrix}
        \tilde\alpha_y^{(t+1)}
     \\
     &=
     \begin{pmatrix}
     I_{d_\Phi} & \alpha_\Phi^\star\\
     0 & \alpha_\Psi^\star
     \end{pmatrix}
     \begin{pmatrix}
     \tilde\alpha_\Phi^{(t+1)} \\
     \tilde\alpha_y^{(t+1)}
     \end{pmatrix}.
     \end{align}
     Combining the two equations above, we have:
     \begin{align}
     \begin{pmatrix}
     \alpha_\Phi^{(t+1)} \\
     \alpha_\Psi^{(t+1)}
     \end{pmatrix} &=
     k^{(t)}
     \begin{pmatrix}
     I_{d_\Phi} & \alpha_\Phi^\star\\
     0 & \alpha_\Psi^\star
     \end{pmatrix}
     \begin{pmatrix}
     I_{d_\Phi} & \beta_\Phi^\star\\
     0 & \beta_y^\star
     \end{pmatrix}
     \begin{pmatrix}
     \alpha_\Phi^{(t)} \\
     \alpha_\Psi^{(t)}
     \end{pmatrix} \\
     &= k^{(t)}
     \begin{pmatrix}
     I_{d_\Phi} & \beta_\Phi^\star + \alpha_\Phi^\star\beta_y^\star \\
     0 & \alpha_\Psi^\star\beta_y^\star
     \end{pmatrix}
     \begin{pmatrix}
     \alpha_\Phi^{(t)} \\
     \alpha_\Psi^{(t)}
     \end{pmatrix}.
     \end{align}
     Thus,
     \begin{align}
     \begin{pmatrix}
     \alpha_\Phi^{(t)} \\
     \alpha_\Psi^{(t)}
     \end{pmatrix} &=
     K^{(t)}
     \begin{pmatrix}
     I_{d_\Phi} & \beta_\Phi^\star + \alpha_\Phi^\star\beta_y^\star \\
     0 & \alpha_\Psi^\star\beta_y^\star
     \end{pmatrix}^t
     \begin{pmatrix}
     \alpha_\Phi^{(0)} \\
     \alpha_\Psi^{(0)}
     \end{pmatrix} \\
     \label{eq:app_convergence_2}
     &=
     K^{(t)}
     \begin{pmatrix}
     I_{d_\Phi} & \beta_\Phi^\star + \alpha_\Phi^\star\beta_y^\star \\
     0 & \alpha_\Psi^\star\beta_y^\star
     \end{pmatrix}^t
     \begin{pmatrix}
     \alpha_\Phi^\star \\
     \alpha_\Psi^\star
     \end{pmatrix}.
     \end{align}
     In the above equation, $K^{(t)}=\prod_{0\leq u<t} k^{(u)}$.

     Define $\hat f^{(t)} = (\hat \alpha_\Phi^{(t)})^T\Phi + (\hat \alpha_\Psi^{(t)})^T\Psi$, where
     \begin{align}
     \label{eq:app_theo_convergence_h2_16}
    \begin{pmatrix}
     \hat \alpha_\Phi^{(t)} \\
     \hat \alpha_\Psi^{(t)}
     \end{pmatrix} =
     \begin{pmatrix}
     I_{d_\Phi} & \beta_\Phi^\star + \alpha_\Phi^\star\beta_y^\star \\
     0 & \alpha_\Psi^\star\beta_y^\star
     \end{pmatrix}^t
     \begin{pmatrix}
     \alpha_\Phi^\star \\
     \alpha_\Psi^\star
     \end{pmatrix}.
     \end{align} 
     In the following we show $\hat \alpha_\Psi^{(t)} \rightarrow 0$. Since $\hat \alpha_\Psi^{(t)} = (\alpha_\Psi^\star\beta_y^\star)^t \alpha_\Psi^\star$, and $\alpha_\Psi^\star\beta_y^\star \in \mathbb R^{d_\Psi\times d_\Psi}$ has exactly one nonzero eigenvalue $\beta_y^\star\alpha_\Psi^\star$, we just have to show $|\beta_y^\star\alpha_\Psi^\star| < 1$.
     \begin{align}
         \beta_y^\star\alpha_\Psi^\star &= 
         ( \Sigma_{yy}- \Sigma_{y\Phi}\Sigma_{\Phi\Phi}^{-1}\Sigma_{\Phi y})^{-1} \\
        &\;\;\;\;( \Sigma_{y\Psi}- \Sigma_{y\Phi}\Sigma_{\Phi\Phi}^{-1}\Sigma_{\Phi\Psi})
        ( \Sigma_{\Psi\Psi}- \Sigma_{\Psi\Phi}\Sigma_{\Phi\Phi}^{-1}\Sigma_{\Phi\Psi})^{-1}
        ( \Sigma_{\Psi y}- \Sigma_{\Psi \Phi}\Sigma_{\Phi \Phi}^{-1}\Sigma_{\Phi y}).
     \end{align}
     Since, $(\Sigma_{yy}- \Sigma_{y\Phi}\Sigma_{\Phi\Phi}^{-1}\Sigma_{\Phi y})$ and $(\Sigma_{\Psi \Psi}- \Sigma_{\Psi \Phi}\Sigma_{\Phi \Phi}^{-1}\Sigma_{\Phi \Psi})$ are both Schur complements of $\Sigma$'s principal submatrix, they are still positive definite.
     
     Thus, 
     \begin{align}
     &\Sigma_{yy}- \Sigma_{y\Phi}\Sigma_{\Phi\Phi}^{-1}\Sigma_{\Phi y} > 0. \\
     &( \Sigma_{y\Psi}- \Sigma_{y\Phi}\Sigma_{\Phi \Phi}^{-1}\Sigma_{\Phi \Psi})
        ( \Sigma_{\Psi \Psi}- \Sigma_{\Psi\Phi}\Sigma_{\Phi\Phi}^{-1}\Sigma_{\Phi\Psi})^{-1}
        ( \Sigma_{\Psi y}- \Sigma_{\Psi\Phi}\Sigma_{\Phi\Phi}^{-1}\Sigma_{\Phi y}) \geq 0.
     \end{align}
     It follows that $\beta_y^\star\alpha_\Psi^\star \geq 0$. So we just have to show $\beta_y^\star\alpha_\Psi^\star < 1$. 

     Since $\det(\Sigma) > 0$, by applying row addition on $\Sigma$, we have:
     \begin{align}
        \det 
         \begin{pmatrix}
         \Sigma_{\Phi\Phi} & \Sigma_{\Phi\Psi} & \Sigma_{\Phi y} \\
         0 & \Sigma_{\Psi\Psi}- \Sigma_{\Psi\Phi}\Sigma_{\Phi\Phi}^{-1}\Sigma_{\Phi\Psi} & \Sigma_{\Psi y}- \Sigma_{\Psi \Phi}\Sigma_{\Phi \Phi}^{-1}\Sigma_{\Phi y} \\
         0 & \Sigma_{y\Psi}- \Sigma_{y\Phi}\Sigma_{\Phi \Phi}^{-1}\Sigma_{\Phi\Psi} & \Sigma_{yy}- \Sigma_{y\Phi}\Sigma_{\Phi\Phi}^{-1}\Sigma_{\Phi y}
         \end{pmatrix} > 0.
     \end{align}
     It follows that:
     \begin{align}
         &(\Sigma_{yy}- \Sigma_{y\Phi}\Sigma_{\Phi\Phi}^{-1}\Sigma_{\Phi y}) \\
         &- ( \Sigma_{y\Psi}- \Sigma_{y\Phi}\Sigma_{\Phi \Phi}^{-1}\Sigma_{\Phi \Psi})
        ( \Sigma_{\Psi \Psi}- \Sigma_{\Psi \Phi}\Sigma_{\Phi \Phi}^{-1}\Sigma_{\Phi \Psi})^{-1}
        ( \Sigma_{\Psi y}- \Sigma_{\Psi \Phi}\Sigma_{\Phi \Phi}^{-1}\Sigma_{\Phi y}) > 0.
     \end{align}
     Rearranging to:
     \begin{align}
          \beta_y^\star\alpha_\Psi^\star &= 
            ( \Sigma_{yy}- \Sigma_{y\Phi}\Sigma_{\Phi \Phi}^{-1}\Sigma_{\Phi y})^{-1} \\
            &\;\;\;\;( \Sigma_{y\Psi}- \Sigma_{y\Phi}\Sigma_{\Phi \Phi}^{-1}\Sigma_{\Phi \Psi})
            ( \Sigma_{\Psi \Psi}- \Sigma_{\Psi \Phi}\Sigma_{\Phi\Phi}^{-1}\Sigma_{\Phi\Psi})^{-1}
            ( \Sigma_{\Psi y}- \Sigma_{\Psi \Phi}\Sigma_{\Phi \Phi}^{-1}\Sigma_{\Phi y}) \\
            &< 1.
     \end{align}
     Thus,  $0 \leq \beta_y^\star\alpha_\Psi^\star < 1$.

     In the following, we show $\hat \alpha_\Phi^{(t)} \rightarrow \Sigma_{\Phi\Phi}^{-1}\Sigma_{\Phi y}$.
     By Equation~\ref{eq:app_theo_convergence_h2_16} and $|\beta_y^\star\alpha_\Psi^\star| < 1$, we have:
     \begin{align}
     \hat\alpha_\Phi^{(t)} &= \alpha_\Phi^\star + (\beta_\Phi^\star + \alpha_\Phi^\star\beta_y^\star)\sum_{0\leq u < t}( \alpha_\Psi^\star\beta_y^\star)^u \alpha_\Psi^\star \\
     &\xrightarrow{t \to \infty}  \alpha_\Phi^\star +  (\beta_\Phi^\star + \alpha_\Phi^\star\beta_y^\star) (I_{d_\Psi} - \alpha_\Psi^\star\beta_y^\star)^{-1}\alpha_\Psi^\star \\
     \label{eq:app_convergence_14}
     &= \alpha_\Phi^\star +  (\beta_\Phi^\star + \alpha_\Phi^\star\beta_y^\star) (1 - \beta_y^\star\alpha_\Psi^\star)^{-1}\alpha_\Psi^\star \\
      &=
     \frac{\alpha_\Phi^\star+\beta_\Phi^\star\alpha_\Psi^\star}{1-\beta_y^\star\alpha_\Psi^\star}.
     \end{align}
     Equation~\ref{eq:app_convergence_14} follows from the fact that $(I_{d_\Psi} - \alpha_\Psi^\star\beta_y^\star)\alpha_\Psi^\star = (1 - \beta_y^\star\alpha_\Psi^\star)\alpha_\Psi^\star$.

     Define $\gamma_\Phi^\star = \Sigma_{\Phi\Phi}^{-1}\Sigma_{\Phi y}$ such that $\mathbb E[Y|\Phi] = (\gamma_\Phi^\star)^T\Phi$. We have
     \begin{align}
         &\;\;\;\;\;\; \mathbb E\left[  \mathbb E\left[ \mathbb E\left[ Y|\Phi,\Psi  \right]  | \Phi,Y \right]  |\Phi \right] = \mathbb E[Y|\Phi]. \\
         &\Leftrightarrow \mathbb E\left[  \mathbb E\left[  (\alpha_\Phi^\star)^T\Phi + (\alpha_\Psi^\star)^T \Psi    | \Phi,Y \right]  |\Phi \right] = (\gamma_\Phi^\star)^T\Phi. \\
         &\Leftrightarrow \mathbb E\left[   (\alpha_\Phi^\star)^T\Phi + (\alpha_\Psi^\star)^T(\beta_\Phi^\star)^T\Phi + (\alpha_\Psi^\star)^T(\beta_y^\star)^TY     |\Phi \right] = (\gamma_\Phi^\star)^T\Phi. \\
          &\Leftrightarrow  (\alpha_\Phi^\star)^T\Phi + (\alpha_\Psi^\star)^T(\beta_\Phi^\star)^T\Phi + (\alpha_\Psi^\star)^T(\beta_y^\star)^T(\gamma_\Phi^\star)^T\Phi     = (\gamma_\Phi^\star)^T\Phi. \\
         &\Leftrightarrow  \alpha_\Phi^\star + \beta_\Phi^\star\alpha_\Psi^\star + \beta_y^\star\alpha_\Psi^\star\gamma_\Phi^\star = \gamma_\Phi^\star. \\
         &\Leftrightarrow \frac{\alpha_\Phi^\star+\beta_\Phi^\star\alpha_\Psi^\star}{1-\beta_y^\star\alpha_\Psi^\star} = \gamma_\Phi^\star.
     \end{align}
     Thus, $\hat \alpha_\Phi^{(t)} \rightarrow \Sigma_{\Phi\Phi}^{-1}\Sigma_{\Phi y}$. Subsequently, $\hat f^{(t)} \rightarrow (\Sigma_{\Phi\Phi}^{-1}\Sigma_{\Phi y})^T \Phi = \mathbb E[Y|\Phi]$. The convergence ratio is $M(\Sigma) = \beta_y^\star\alpha_\Psi^\star$.

    We have:
    \begin{align}
        \begin{pmatrix}
     \tilde\alpha_\Phi^{(t+1)} \\
     \tilde\alpha_y^{(t+1)}
     \end{pmatrix} 
     &= 
     K^{(t+1)} 
      \begin{pmatrix}
     I_{d_\Phi} & \beta_\Phi^\star\\
     0 & \beta_y^\star
     \end{pmatrix}
     \begin{pmatrix}
     \hat \alpha_\Phi^{(t)} \\
     \hat \alpha_\Psi^{(t)}
     \end{pmatrix},
     \end{align}
     where
     \begin{align}
    \begin{pmatrix}
     I_{d_\Phi} & \beta_\Phi^\star\\
     0 & \beta_y^\star
     \end{pmatrix}
     \begin{pmatrix}
     \hat \alpha_\Phi^{(t)} \\
     \hat \alpha_\Psi^{(t)}
     \end{pmatrix}   
     & \xrightarrow{t \to \infty} 
      \begin{pmatrix}
     I_{d_\Phi} & \beta_\Phi^\star\\
     0 & \beta_y^\star
     \end{pmatrix}
     \begin{pmatrix}
     \Sigma_{\Phi\Phi}^{-1}\Sigma_{\Phi y} \\
     0
     \end{pmatrix} \\
     &= 
     \begin{pmatrix}
     \Sigma_{\Phi\Phi}^{-1}\Sigma_{\Phi y} \\
     0
     \end{pmatrix}.
    \end{align}
    Define $\hat {\tilde f}^{(t)} = (\hat{\tilde \alpha}_\Phi^{(t)})^T \Phi + (\hat{\tilde \alpha}_y^{(t)})^T Y$, where
    \begin{align}
        \begin{pmatrix}
         \hat{\tilde\alpha}_\Phi^{(t+1)} \\
         \hat{\tilde\alpha}_y^{(t+1)}
         \end{pmatrix} 
         =
          \begin{pmatrix}
     I_{d_\Phi} & \beta_\Phi^\star\\
     0 & \beta_y^\star
     \end{pmatrix}
     \begin{pmatrix}
     \hat \alpha_\Phi^{(t)} \\
     \hat \alpha_\Psi^{(t)}
     \end{pmatrix}.
    \end{align}
    Thus, $\hat {\tilde f}^{(t)} \rightarrow \mathbb E[Y|\Phi]$.  Since $\tilde f^{(t)} = K^{(t)} \hat {\tilde f}^{(t)}$, we have $\tilde f^{(t)} = \mathbb E[Y|\hat {\tilde f}^{(t)}] \rightarrow \mathbb E[Y|\Phi]$.

    Subsequently, $f^{(t)} = \mathbb E[\tilde  f^{(t)} | \Phi, \Psi] \rightarrow \mathbb E[Y|\Phi]$.

\end{proof}


\section{Limitations}
\label{sec:app_limit}
Both our theory and algorithm focuses on the bounded regression setting. The definition of extended multicalibration does not depend on the risk function. However, the analysis of the maximal grouping function class as a linear space assumes a continuous probability distribution of observations, implying a continuous target domain. The convergence of MC-Pseudolabel is also established in a regression setting. All experiments are performed on regression tasks. As most algorithms for out-of-distribution generalization are set up with classification problems, we fill the gap for regression and leave an extension to general risk functions for future work.